\documentclass{article}
\pdfoutput=1

\usepackage[final, nonatbib]{neurips_2021}

\usepackage[utf8]{inputenc} 
\usepackage[T1]{fontenc}    
\usepackage[pdftex]{hyperref}       
\usepackage{url}            
\usepackage{booktabs}       
\usepackage{amsfonts}       
\usepackage{nicefrac}       
\usepackage{microtype}      
\usepackage{xcolor}         
\usepackage[natbib=true]{biblatex}
\usepackage{freetikz}
\usepackage{rotating}
\usetikzlibrary{positioning}
\usetikzlibrary{decorations.markings}
\usetikzlibrary{shapes,snakes}
\tikzset{->-/.style={decoration={
  markings,
  mark=at position .5 with {\arrow{>}}},postaction={decorate}}}

\usepackage{amsmath}
\usepackage{mathtools}
\usepackage{amssymb}
\usepackage{amsthm}
\usepackage{algorithm,algorithmic}
\usepackage{bbm}
\usepackage{enumitem}

\title{Near-Optimal Multi-Perturbation Experimental Design for Causal Structure Learning}

\DeclareMathOperator*{\E}{\mathbb{E}}
\DeclareMathOperator*{\Prob}{\mathbb{P}}

\DeclarePairedDelimiter\abs{\lvert}{\rvert}%
\DeclareMathOperator*{\argmax}{arg\,max}

\newcommand{\bigO}{\mathcal{O}}

\newtheorem{theorem}{Theorem}

\newtheorem{lemma}{Lemma}
\newtheorem{proposition}{Proposition}

\theoremstyle{definition}
\newtheorem{definition}{Definition}

\newtheorem{objective}{Objective}

\newcommand{\scglong}{\textsc{Double Greedy Continuous}}
\newcommand{\scg}{\textsc{DGC}}
\newcommand{\ssglong}{\textsc{Separating System Greedy}}
\newcommand{\ssg}{\textsc{SSG}}
\newcommand{\csslong}{\textsc{Continuous Separating System Greedy}}
\newcommand{\css}{\textsc{CSSG}}

\newcommand{\Fmi}{\ensuremath{F_{\text{MI}}}}
\newcommand{\Finf}{\ensuremath{F_{\infty}}}
\newcommand{\apFinf}{\ensuremath{\tilde{F}_{\infty}}}

\newcommand{\Feo}{\ensuremath{F_{\text{EO}}}}
\newcommand{\feo}{\ensuremath{f_{\text{EO}}}}

\newcommand{\rand}{\textsc{Rand}}
\newcommand{\ssga}{\textsc{SSG-a}}
\newcommand{\ssgb}{\textsc{SSG-b}}
\newcommand{\greedy}{\textsc{Greedy}}
\newcommand{\ssgbinf}{\textsc{SSG-b}-$\infty$}
\newcommand{\mytest}{\ifmmode \mathrm{Yes}\else No\fi.}
\author{%
  Scott Sussex \\ 
  Department of Computer Science\\ 
  ETH Z{\"u}rich\\
  Z{\"u}rich, Switzerland \\
  \texttt{scott.sussex@inf.ethz.ch} \\
  \And
  Andreas Krause \\
  Department of Computer Science\\ 
  ETH Z{\"u}rich\\
  Z{\"u}rich, Switzerland \\
  \And
  Caroline Uhler \\
  Laboratory for Information \& Decision Systems \\ 
  Massachusetts Institute of Technology\\
  Cambridge, MA \\
}

\begin{document}

\maketitle

\begin{abstract}
 \looseness -1 Causal structure learning is a key problem in many domains. Causal structures can be learnt by performing experiments on the system of interest. We address the largely unexplored problem of designing a batch of experiments that each {\em simultaneously intervene on multiple variables}. While potentially more informative than the commonly considered single-variable interventions, selecting such interventions is algorithmically much more challenging, due to the doubly-exponential combinatorial search space over sets of composite interventions.
In this paper, we develop efficient algorithms for optimizing different objective functions quantifying the informativeness of a budget-constrained batch of experiments.
By establishing novel submodularity properties of these objectives, we provide approximation guarantees for our algorithms. 
Our algorithms empirically perform superior to both random interventions and algorithms that only select single-variable interventions. 
\end{abstract}

\section{Introduction}
\label{sec:intro}

The problem of finding the causal relationships between a set of variables is ubiquitous throughout the sciences. For example, scientists are interested in reconstructing gene regulatory networks (GRNs) of biological cells \cite{gardner2003inferring}. Directed Acyclic Graphs (DAGs) are a natural way to represent causal structures, with a directed edge from variable $X$ to $Y$ representing $X$ being a direct cause of $Y$ \cite{Spirtes:1607850}. 

\looseness -1 Learning the causal structure of a set of variables is fundamentally difficult. With only observational data, in general we can only identify the true DAG up to a set of DAGs called its {\em Markov Equivalence Class (MEC)} \cite{verma1990equivalence}. Empirically, for sparse DAGs the size of the MEC grows exponentially in the number of nodes \cite{he2015counting}. Identifiability can be improved by intervening on variables, meaning one perturbs a subset of the variables and then observes more samples from the system \cite{Eberhardt2007,hauser2012characterization,yang2018characterizing}. There exist various inference algorithms for learning causal structures from a combination of observational and interventional data \cite{hauser2012characterization,wang2017permutation, yang2018characterizing,mooij2020,squires2020}. Here we focus on the identification of DAGs that have no unobserved confounding variables. 

\looseness -1 Performing experiments is often expensive, however.  Thus, we are interested in learning as much about the causal structure as possible given some constraints on the interventions. In this work, we focus on the {\em batched setting}, where several interventions are performed in parallel. This is a natural setting in scientific domains like reconstructing GRNs. Existing works propose meaningful objective functions for this batched causal structure learning problem and then give algorithms that have provable guarantees \cite{agrawal2019abcd,ghassami2018budgeted}. However, these works focus on the setting where only a {\em single} random variable is perturbed per intervention. It is an open question as to whether there exist efficient algorithms for the {\em multiple-perturbation} setting, where more than one variable is perturbed in each intervention.

\looseness -1 For the task of reconstructing GRNs, it is now possible for experimenters to perturb multiple genes in a single cell \cite{adamson2016multiplexed, dixit2016perturb}. Figures~\ref{fig:over} b) and c) illustrate a specific example where a two-node intervention completely identifies a DAG in half as many interventions as single-node interventions. In general, it is possible for a set of $q$-node interventions to orient {\em up to $q$-times more} edges in a DAG than single-node interventions (see the supplementary material for a more general example). While multi-perturbation interventions can be more informative, designing them is algorithmically challenging because it leads to an {\em exponentially} larger search space: any algorithm must now select a set of sets.  

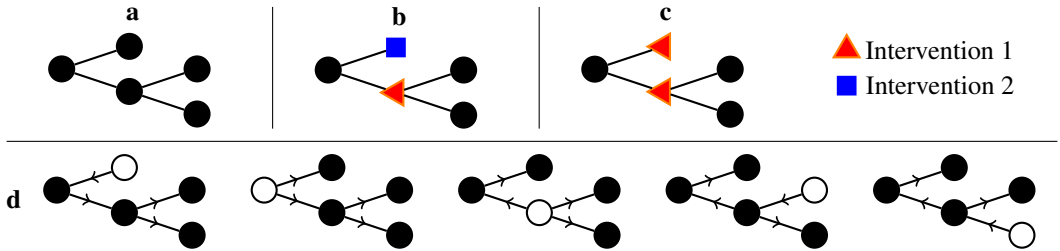
\begin{figure*}[t]
\minipage{0.24\textwidth}
\centering
\textbf{a}
\vskip 0.05in
\begin{turn}{90}
\begin{tikzpicture}[scale = 0.6, roundnode/.style={circle, draw=black!100, fill=black!1000, thick, minimum size=2mm},
]
\node[roundnode] (d0) at (1.5, 7.5) {};
\node[roundnode] (d1) at (1, 6) {};
\node[roundnode] (d2) at (2, 6) {};
\node[roundnode] (d3) at (0.5, 4.5) {};
\node[roundnode] (d4) at (1.5, 4.5) {};
  
\draw[-, thick] (d0) -- (d1);
\draw[-, thick] (d0) -- (d2);
\draw[-, thick] (d1) -- (d3);
\draw[-, thick] (d1) -- (d4);
\end{tikzpicture}
\end{turn}
\endminipage\hfill \vline
\minipage{0.24\textwidth}
\centering
\textbf{b}
\vskip 0.05in
\begin{turn}{90}
\begin{tikzpicture}[scale = 0.6, roundnode/.style={circle, draw=black!100, fill=black!1000, thick, minimum size=2mm},
bluenode/.style={rectangle, draw=blue!100, fill=blue!1000, thick, minimum size=2mm},
orangenode/.style={regular polygon,regular polygon sides=3, draw=orange!100, fill=orange!1000, thick, minimum size=2mm, scale=0.6},
]
\node[roundnode] (d0) at (1.5, 7.5) {};
\node[orangenode] (d1) at (1, 6) {};
\node[bluenode] (d2) at (2, 6) {};
\node[roundnode] (d3) at (0.5, 4.5) {};
\node[roundnode] (d4) at (1.5, 4.5) {};
  
\draw[-, thick] (d0) -- (d1);
\draw[-, thick] (d0) -- (d2);
\draw[-, thick] (d1) -- (d3);
\draw[-, thick] (d1) -- (d4);
\end{tikzpicture}
\end{turn}
\endminipage\hfill \vline
\minipage{0.24\textwidth}
\centering
\textbf{c}
\vskip 0.05in
\begin{turn}{90}
\begin{tikzpicture}[scale = 0.6, roundnode/.style={circle, draw=black!100, fill=black!1000, thick, minimum size=2mm},
bluenode/.style={rectangle, draw=blue!100, fill=blue!1000, thick, minimum size=2mm},
orangenode/.style={regular polygon,regular polygon sides=3, draw=orange!100, fill=orange!1000, thick, minimum size=2mm, scale=0.6},
]
\node[roundnode] (d0) at (1.5, 7.5) {};
\node[orangenode] (d1) at (1, 6) {};
\node[orangenode] (d2) at (2, 6) {};
\node[roundnode] (d3) at (0.5, 4.5) {};
\node[roundnode] (d4) at (1.5, 4.5) {};
  
\draw[-, thick] (d0) -- (d1);
\draw[-, thick] (d0) -- (d2);
\draw[-, thick] (d1) -- (d3);
\draw[-, thick] (d1) -- (d4);
\end{tikzpicture}
\end{turn}
\endminipage\hfill
\minipage{0.24\textwidth}
\centering
\begin{tikzpicture}[scale = 0.9, roundnode/.style={circle, draw=black!100, fill=black!1000, thick, minimum size=2mm},
bluenode/.style={rectangle, draw=blue!100, fill=blue!1000, thick, minimum size=2mm},
orangenode/.style={regular polygon,regular polygon sides=3, draw=orange!100, fill=orange!1000, thick, minimum size=2mm, scale=0.6},
]
\matrix {
  \node [orangenode,label=right:Intervention 1] {}; \\
  \node [bluenode,label=right:Intervention 2] {}; \\
};
\end{tikzpicture}
\endminipage\hfill
\vskip 0.05in

\hrulefill
\vskip 0.05in

\textbf{d}
\minipage{0.19\textwidth}
\centering
\begin{turn}{90}
\begin{tikzpicture}[scale = 0.6, roundnode/.style={circle, draw=black!100, fill=black!100, thick, minimum size=2mm},
rootnode/.style={circle, draw=black!100, fill=black!0, thick, minimum size=2mm},
]
\node[roundnode] (d0) at (1.5, 7.5) {};
\node[roundnode] (d1) at (1, 6) {};
\node[rootnode] (d2) at (2, 6) {};
\node[roundnode] (d3) at (0.5, 4.5) {};
\node[roundnode] (d4) at (1.5, 4.5) {};

\draw[->-, thick] (d0) -- (d1);
\draw[->-, thick] (d2) -- (d0);
\draw[->-, thick] (d1) -- (d3);
\draw[->-, thick] (d1) -- (d4);
\end{tikzpicture}
\end{turn}
\endminipage\hfill 
\minipage{0.19\textwidth}
\centering
\begin{turn}{90}
\begin{tikzpicture}[scale = 0.6, roundnode/.style={circle, draw=black!100, fill=black!100, thick, minimum size=2mm},
rootnode/.style={circle, draw=black!100, fill=black!0, thick, minimum size=2mm},
]
\node[rootnode] (d0) at (1.5, 7.5) {};
\node[roundnode] (d1) at (1, 6) {};
\node[roundnode] (d2) at (2, 6) {};
\node[roundnode] (d3) at (0.5, 4.5) {};
\node[roundnode] (d4) at (1.5, 4.5) {};
  
\draw[->-, thick] (d0) -- (d1);
\draw[->-, thick] (d0) -- (d2);
\draw[->-, thick] (d1) -- (d3);
\draw[->-, thick] (d1) -- (d4);
\end{tikzpicture}
\end{turn}
\endminipage\hfill 
\minipage{0.19\textwidth}
\centering
\begin{turn}{90}
\begin{tikzpicture}[scale = 0.6, roundnode/.style={circle, draw=black!100, fill=black!100, thick, minimum size=2mm},
rootnode/.style={circle, draw=black!100, fill=black!0, thick, minimum size=2mm},
]
\node[roundnode] (d0) at (1.5, 7.5) {};
\node[rootnode] (d1) at (1, 6) {};
\node[roundnode] (d2) at (2, 6) {};
\node[roundnode] (d3) at (0.5, 4.5) {};
\node[roundnode] (d4) at (1.5, 4.5) {};
 
\draw[->-, thick] (d1) -- (d0);
\draw[->-, thick] (d0) -- (d2);
\draw[->-, thick] (d1) -- (d3);
\draw[->-, thick] (d1) -- (d4);
\end{tikzpicture}
\end{turn}
\endminipage\hfill
\minipage{0.19\textwidth}
\centering
\begin{turn}{90}
\begin{tikzpicture}[scale = 0.6, roundnode/.style={circle, draw=black!100, fill=black!100, thick, minimum size=2mm},
rootnode/.style={circle, draw=black!100, fill=black!0, thick, minimum size=2mm},
]
\node[roundnode] (d0) at (1.5, 7.5) {};
\node[roundnode] (d1) at (1, 6) {};
\node[roundnode] (d2) at (2, 6) {};
\node[roundnode] (d3) at (0.5, 4.5) {};
\node[rootnode] (d4) at (1.5, 4.5) {};

\draw[->-, thick] (d1) -- (d0);
\draw[->-, thick] (d0) -- (d2);
\draw[->-, thick] (d1) -- (d3);
\draw[->-, thick] (d4) -- (d1);
\end{tikzpicture}
\end{turn}
\endminipage\hfill 
\minipage{0.19\textwidth}
\centering
\begin{turn}{90}
\begin{tikzpicture}[scale = 0.6, roundnode/.style={circle, draw=black!100, fill=black!100, thick, minimum size=2mm},
rootnode/.style={circle, draw=black!100, fill=black!0, thick, minimum size=2mm},
]
\node[roundnode] (d0) at (1.5, 7.5) {};
\node[roundnode] (d1) at (1, 6) {};
\node[roundnode] (d2) at (2, 6) {};
\node[rootnode] (d3) at (0.5, 4.5) {};
\node[roundnode] (d4) at (1.5, 4.5) {};
 
\draw[->-, thick] (d1) -- (d0);
\draw[->-, thick] (d0) -- (d2);
\draw[->-, thick] (d3) -- (d1);
\draw[->-, thick] (d1) -- (d4);
\end{tikzpicture}
\end{turn}
\endminipage\hfill

\begin{center}
\caption{\textbf{a)} We illustrate the MEC of a tree graph on $5$ nodes. \textbf{b)} Two single-node interventions are required to fully identify the true DAG. \textbf{c)} Only one two-node intervention is required to fully identify the true DAG. \textbf{d)} This MEC contains $5$ DAGs, each corresponding to a different root node (marked white). This is a property particular to tree MECs.}
\label{fig:over}
\end{center}
\vskip -0.3in
\end{figure*}

\looseness -1 Our main contribution is to provide efficient algorithms for different objective functions with accompanying performance bounds. We demonstrate empirically on both synthetic and GRN graphs that our algorithms result in greater identifiability than existing approaches  that do not make use of multiple perturbations \cite{ghassami2018budgeted, yang2018characterizing}, as well as a random strategy.

\looseness -1 We begin by introducing the notation and the objective functions considered in this work in Section~\ref{sec:background}, before reviewing related work in Section~\ref{sec:related}. In Section~\ref{sec:method} we present our algorithms along with proofs of their performance guarantees. Finally, in Section~\ref{sec:experiments} we demonstrate the superior empirical performance of our method over existing baselines on both synthetic networks and on data generated from models of real GRNs. 

\section{Background and Problem Statement}
\label{sec:background}

\paragraph{Causal DAGs} \looseness -1 Consider a causal DAG $G = ([p],E)$ where $[p] := \{1,...,p\}$ is a set of nodes and $E$ is a set of directed edges. Let $(i, j) \in E$ iff there is an edge from node $i$ to node $j$. 
Each node $i$ is associated with a random variable $X_i$. 
In the GRN example, $X_i$ would be the measurement of the gene expression level for gene $i$. An edge from $i \rightarrow j$ would represent gene $i$ having a causal effect on the expression of gene $j$. The functional dependence of a random variable on its parents can be described by a \emph{structural equation model} (SEM). 

 \looseness -1 The probability distribution over $X=(X_1, \dots , X_p)$ is related to $G$ by the Markov property, meaning each variable $X_i$ is conditionally independent of its non-descendants given its parents \cite{Spirtes:1607850}. From conditional independence tests one can determine the MEC of $G$, a set of DAGs with the same conditional independancies between variables. All members of the MEC share the same undirected skeleton and colliders \cite{verma1990equivalence}. A collider is a pair of edges $(i, k), (j, k) \in E$ such that $(i, j), (j, i) \notin E$. The {\em essential graph} of $G$, $\text{Ess}(G)$, is a partially directed graph, with directed edges where all members of the MEC share the same edge direction, and with undirected edges otherwise \cite{andersson1997characterization}. $\text{Ess}(G)$ uniquely represents the MEC of $G$. These MECs can be large, so we seek to perform interventions on the nodes to reduce the MEC to a smaller set of possible DAGs. 

\vspace{-2mm} \paragraph{Interventions} \looseness -1 We use the term {\em intervention} to refer to a set $I \subset [p]$ of perturbation targets (variables). We assume all interventions are {\em hard interventions}, meaning intervening on a set $I$ removes the incoming edges to the random variables $X_I := (X_i)_{i\in I}$ and sets their joint distribution to some interventional distribution $\mathcal{P}^{I}$ \cite{eberhardt2005number}. In the GRN reconstruction example this corresponds to, for example, running an experiment where we knockout all genes in set $I$. Some of our results extend easily to the alternative model of {\em soft} interventions \cite{markowetz2005probabilistic}, as we discuss in the supplementary material. 

We use $\mathcal{I} = 2^{[p]}$ to refer to the set of all possible interventions. Our goal will be to select a  batch  $\xi$ of interventions, where $\xi$ is a multiset of some $I \in \mathcal{I}$. For practical reasons, we typically have {\em constraints} on the number of interventions, i.e., $\abs{\xi} \leq m$ and on the number of variables involved in each intervention $\abs{I} \leq q $, $\forall I \in \xi$. Namely, there are at most $m$ interventions per batch and each intervention contains at most $q$ nodes. A constraint on the number of perturbations per intervention is natural in reconstructing GRNs, since perturbing too many genes in one cell will leave it unlikely to survive. We refer to the set of $\xi$s satisfying these constraints as $C_{m, q}$. The observational distribution (no intervention) is given by $\xi = \emptyset$. 

\looseness -1 For any set of interventions $\xi$ and DAG $G$, there is a set of $\xi$-Markov equivalent DAGs. These are the set of DAGs that have the same set of conditional independencies under all $I \in \xi$ and under the observational distribution. This set of DAGs is no larger than the MEC of $G$ and can similarly be characterized by an essential graph $\textrm{Ess}^{\xi}(G)$ \cite{hauser2012characterization}. 

\looseness -1 We will always assume that there exist no unobserved common causes of any pair of nodes in $G$. We also assume that the distribution of the random variables satisfies faithfulness with respect to $G$ \cite{Spirtes:1607850}. 

\vspace{-2mm} \paragraph{Choosing optimal interventions} We seek to maximize an objective function $F$ that quantifies our certainty in  the true DAG. In general our goal is to determine
$$\argmax_{\xi \in C_{m, q}} F(\xi).$$
A natural choice for $F$ is given by \citet{agrawal2019abcd}. They assume that there exist parameters $\theta$ that determine the functional relationships between random variables. For example, this could be the coefficients in a linear model. Given existing data $D$, we try to choose $\xi$ that maximizes

\begin{objective}[\bf Mutual Information (MI)]
\label{obj:abcdobj}
\begin{equation}
\begin{split}
\Fmi(\xi) =  {\E}_{G\mid D} {\E}_{y\mid G, \hat{\theta}, \xi}\left[\tilde{U}_{M.I} (y,\xi ;D) \right],
\end{split}
\end{equation}
\end{objective}
where $y$ is the set of samples from the interventions, $\hat{\theta}$ is the current estimate of the parameters given $D$ and $G$, and $\tilde{U}_{M.I} (y,\xi ;D)$ is the mutual information between the posterior over $G$ and the samples $y$. Each intervention produces one sample in $y$. The use of mutual information means the objective aims to, in expectation over all observed samples and true DAGs, minimize the entropy of the posterior distribution over DAGs. There already exist a number of algorithms for determining the posterior over DAGs from observational or experimental data \cite{yang2018characterizing, wang2017permutation, hauser2012characterization}. 

\vspace{-2mm}
\paragraph{Infinite sample objectives} Finding algorithms that optimize the MI objective is difficult because we have to account for noisy observations and limited samples. To remove this complexity, we study the limiting case of {\em infinitely many samples} per unique intervention. The constraints given by $C_{m, q}$ still stipulate that there can be only $m$ unique interventions, but each intervention can be performed with an infinite number of samples. We also assume that an essential graph is already known (i.e., we have infinite observational samples and infinite samples for any experiments performed so far). For objectives with infinite samples per intervention, we treat $\xi$ as a set of interventions, not a multiset, since there is no change in objective value for choosing an intervention twice. Consider $\xi'$ to be the set of interventions contained in our dataset before our current batch. In this setting, maximizing Objective~\ref{obj:abcdobj} reduces to maximizing

\begin{objective}[\bf Mutual info.~inf.~samples (MI-$\infty$)]
\label{obj:infABCDobj}
\begin{equation}
    F_{\infty}(\xi)= - \frac{1}{\abs{\mathcal{G}}} \sum_{G\in\mathcal{G}} \log_2 \abs{\textrm{Ess}^{\xi \cup \xi'}(G)},
\end{equation}
\end{objective}

\looseness -1 where $\textrm{Ess}^{\xi \cup \xi'}(G)$ refers to the updated essential graph after performing interventions in $\xi$. The objective aims to, on average across possible true DAGs, minimize the $\log$ of the essential graph size after performing the interventions. The derivation of this objective is given in the supplementary material. 

\citet{ghassami2018budgeted} study a different objective in the infinite-sample setting. The objective seeks to, on average across possible $G$ given the current essential graph, orient as many edges as possible. Let $R(\xi, G, \xi')$ be the set of edges oriented by $\xi$ if the true DAG is $G$, and the essential graph is given by the $\xi'$-MEC.

\begin{objective}[\bf Edge-orientation (EO)]
\label{obj:edgeorientobj}
\begin{equation}
\begin{split}
    F_{\text{EO}}(\xi)= \frac{1}{\abs{\mathcal{G}}} \sum_{G\in\mathcal{G}}\abs{R(\xi, G, \xi')}.
    \end{split}
\end{equation}
\end{objective}

\looseness -1 The function $R$ is computed as follows. Firstly, $\forall I \in \xi$, orient undirected edge $i - j$ in $\textrm{Ess}^{\xi'}(G)$ if $i \in I$ but $j \notin I$ or vice-versa. Secondly, execute the {\em Meek Rules} \cite{meek1995casual}, which allow inferring additional edge orientations (discussed in the supplementary material) on the resulting partially directed graph. Finally, output the set of all edges oriented.  \citet{agrawal2019abcd} show that this objective is not consistent; however, this is because they fix $\mathcal{G}$ to be the MEC instead of using the most up-to-date essential graph for each batch. In the supplementary material, we show that the version of the objective we work with is indeed consistent. We will drop the dependence of $R$ on $\xi'$ for readability. 

Below, we provide algorithms with near-optimality guarantees for Objectives~\ref{obj:infABCDobj}~and~\ref{obj:edgeorientobj}, while motivating a practical algorithm for Objective \ref{obj:abcdobj}.   

\section{Related Work}
\label{sec:related}

Causality has been widely studied in machine learning \cite{pearl2009causality, peters2017elements}. Here we focus on prior research that is most relevant to our work. 

\looseness -1 \citet{agrawal2019abcd} and \citet{ghassami2018budgeted} give near-optimal greedy algorithms for Objectives~\ref{obj:abcdobj} and~\ref{obj:edgeorientobj} respectively. \citet{ahmaditeshnizi20a} present a dynamic programming algorithm for an adversarial version of Objective~\ref{obj:edgeorientobj}, optimizing for the worst case ground truth DAG in the MEC. However, all these algorithms only apply to {\em single-perturbation} interventions.  Both of these works use the submodularity of the two objectives. In this paper we address the exponentially large search space that arises when designing multi-perturbation interventions, a strictly harder problem. 

\looseness -1 Much existing work in experimental design for causal DAGs is focused on identifying the graph uniquely, while minimizing some cost associated with doing experiments \cite{eberhardt2005number,hyttinen2013experiment,shanmugam2015learning,kocaoglu2017cost, lindgren2018experimental}. When the MEC is large and the number of experiments is small, identifying the entire graph will be infeasible. Instead, one must select interventions that optimize a measure of the information gained about the causal graph.

\citet{lindgren2018experimental} show NP-hardness for selecting an intervention set of at most $m$ interventions, with minimum number of perturbations, that completely identifies the true DAG. This, however, does not directly imply a hardness result for our problem. 

\citet{gamella2020active} propose an approach to experimental design for causal structure learning based on invariant causal prediction. While our approach has guarantees for objectives relating to either the whole graph or functions of the oriented edges, their work is specific to the problem of learning the direct causes of one variable.  

\looseness -1 \citet{acharya2018learning} consider {\em testing} between two candidate causal models. However, the setting differs from ours: they assume the underlying DAG is known but allow for unobserved confounding variables. 

\looseness -1 Designing multi-perturbation interventions has been previously studied in linear cyclic networks, with a focus on parameter identification \cite{gross20}. Here we focus on causal graph identification in DAGs. 

\section{Greedy Algorithms for Experiment Design}
\label{sec:method}
 
All of our algorithms follow the same general strategy. Like in previous works on single-perturbation experimental design \cite{agrawal2019abcd, ghassami2018budgeted}, we greedily add interventions to our intervention set. We add $I$ maximizing $F(\xi \cup \{I\})$ where $\xi$ is the currently proposed set of interventions. This greedy selection is justified because our objectives are submodular, a property we define formally later. For single-perturbation experimental design, this is algorithmically simple since there are only $p$ possible interventions. However, for multi-perturbation interventions even selecting greedily is {\em intractable} at scale since we have $\binom{p}{q}$ possible interventions. Therefore we provide ways to find an intervention that is approximately greedy, i.e, an intervention with marginal improvement in objective that is close to that of the greedy intervention. 

A further challenge with the greedy approach is that it involves evaluating the objective, which for our objectives is a potentially {\em exponential sum} over members of an essential graph. Each of the two algorithms we give has a different strategy for overcoming this. 

In Section~\ref{subsec:method:orient}, we present \scglong{} (\scg) for optimizing Objective~\ref{obj:edgeorientobj}, the edge-orientation objective. For greedily selecting interventions to maximize an exponential sum, we employ the stochastic continuous optimization technique of \citet{hassani2019stochastic}. 

In Section~\ref{subsec:method:MI}, we present \ssglong{} (\ssg) for optimizing Objective~\ref{obj:infABCDobj}, MI-$\infty$. To greedily select interventions, we use the construction of \emph{separating systems} (SS) \cite{wegener1979separating, shanmugam2015learning, lindgren2018experimental}, to create a smaller set of interventions to search over. Collectively, the interventions in the SS fully orient the graph. To handle tractably evaluating the objective, we use the idea of \citet{ghassami2018budgeted} and \citet{agrawal2019abcd} to optimize an approximation of the objective constructed using a limited sample of DAGs. 

\looseness -1 To give near-optimality guarantees for these algorithms, we will use two properties of the objectives: monotonicity and submodularity. 
 
\begin{definition} 
\label{def:monotone}
A set function $F:2^V\rightarrow \mathbb{R}$ is \emph{monotonically increasing} if for all sets $I_1 \subseteq I_2 \subseteq V$ we have $F(I_1)\leq F(I_2)$. 
\end{definition}

\begin{definition} 
\label{def:submod}
A set function $F:2^V\rightarrow \mathbb{R}$ is \emph{submodular} if for all sets $I_1 \subseteq I_2 \subseteq V$ and all $v \in V \setminus I_2$ we have $F(I_1 \cup \{v\}) - F(I_1) \geq F(I_2 \cup \{v\}) - F(I_2)$. 
\end{definition}

Submodularity is a natural diminishing returns property, and many strategies have been studied for optimizing submodular objectives \cite{krause2014submodular}. In both the above definitions, $V$ is called the \emph{groundset}, the set that we can choose elements from. In the single-perturbation problem, the groundset is just $[p]$, whereas in our case it is all subsets of up to $q$ nodes. 

\looseness -1 We show that \scg{} achieves an objective value within a constant factor of the optimal intervention set on Objective~\ref{obj:edgeorientobj}. \ssg{} does not achieve a constant-factor guarantee, but for both infinite sample objectives we obtain a lower bound on its performance. 

All of our algorithms run in polynomial time; however, they assume access to a uniform sampler across all DAGs in the essential graph. This exists for sampling from the MEC \cite{wienobst2020polynomialtime} but not for essential graphs given existing interventions. In practice, we find that an efficient non-uniform sampler \cite{ghassami2018budgeted} can be used to achieve strong empirical performance. 


\subsection{Optimizing the Edge-orientation Objective}
\label{subsec:method:orient}

In the following, we develop an algorithm for maximizing Objective~\ref{obj:edgeorientobj}.

In fact, the algorithm we provide has a near-optimality guarantee for a more general form of $\Feo$, namely
$$\Feo(\xi) = \sum_{G \in \mathcal{G}} a(G) \sum_{e\in G} w(e) \mathbbm{1}(e \in R(\xi, G)),$$ 
where $\forall e, w(e)\geq 0$ and $\forall G, a(G) \geq 0$. The weights $a(G)$ can be thought of as corresponding to having a non-uniform prior over the DAGs in the essential graph, whilst the weights $w(e)$ can be thought of as assigning priority to the orienting of certain edges. The inner sum above is a {\em weighted coverage function} \cite{krause2014submodular} over the set of edges. 

\looseness -1 We will first show that $\Feo$ is monotone submodular over groundset $\mathcal{I}$. This generalizes a result by \citet{ghassami2018budgeted} who showed the same result for groundset $[p]$ (single perturbation interventions). 

\begin{lemma}
\label{lem:F_mon_sub}
\Feo{} is monotone submodular over the groundset~$\mathcal{I}$.
\end{lemma}
\begin{proof}
All proofs are presented in the supplementary material unless otherwise stated. 
\end{proof}

As mentioned, we cannot use a greedy search directly since the groundset $\mathcal{I}$ is too large. Instead, we develop an algorithm for selecting an intervention with near-maximal utility compared to the greedy choice. In particular, our strategy is to prove a submodularity result over the function $F$ with modified domain. Consider the set function $\Feo^{\xi}(I) = \Feo(\xi \cup \{I\})$ for fixed $\xi$. 

\begin{lemma}
$\Feo^{\xi}$ is non-monotone submodular over the groundset $[p]$.
\end{lemma}

The Non-monotone Stochastic Continuous Greedy (\textsc{NMSCG}) algorithm of \citet{mokhtari2020stochastic} can therefore be used as a subroutine to select, in expectation, an approximately greedy intervention to add to an existing intervention set. The algorithm uses a stochastic gradient-based method to optimize a continuous relaxation of our objective, and then rounds the solution to obtain a set of interventions. The continuous relaxation of $\Feo^{\xi}$ is the multilinear extension
$$\feo^{\xi}(x) = \sum_{I \in \mathcal{I}} \Feo^{\xi}(I) \prod_{i \in I} x_i \prod_{i \notin I} (1-x_i)$$
with constraints $\sum_i x_i \leq q $, $0 \leq x_i \leq 1$ for all nodes $i$. The multilinear extension can be thought of as computing the expectation of $\Feo^{\xi}(I)$, when input $x$ is a vector of independent probabilities such that $x_i$ is the probability of including node $i$ in the intervention.  The sum over $\mathcal{I}$ in $\feo^{\xi}$ and the sum over DAGs in $\Feo^{\xi}$ make computing the gradient of this objective intractable. Therefore, we compute an unbiased stochastic approximation of the gradient $\nabla \feo^{\xi}(x)$ by uniformly sampling a DAG $G$ from $\mathcal{G}$ and intervention $I$ from the distribution specified by $x$. Define
\begin{equation}
\label{eqn:f-R}
\hat{f}_{\text{EO}}^{\xi}(I, G) = \abs{R(\xi \cup \{I\}, G)}.
\end{equation}
\looseness -1 \citet{mokhtari2020stochastic} show that an unbiased estimate of the gradient of $f(x)$ can be computed by sampling $G$ and $I$ to approximate
\begin{equation}
\label{eqn:grad}
\frac{\partial}{\partial x_i} \feo^{\xi}(x)   = \E_{G, I \mid x}\left[  \hat{f}_{\text{EO}}^{\xi}(I, G; I_i \leftarrow 1) - \hat{f}_{\text{EO}}^{\xi}(I, G; I_i \leftarrow 0) \right],
\end{equation}
where $I_i \leftarrow 0$ means that if $i\in I$, remove it. The use of a stochastic gradient means that $\Feo$ can be efficiently optimized despite it being a possibly exponential sum over~$\mathcal{G}$.

After several gradient updates we obtain a vector of probabilities $x$ that approximately maximizes $\feo$. To obtain an intervention $I$ one uses a \textsc{Round} function, for example pipage rounding which, on submodular functions, has the guarantee that $\E[\Feo^{\xi}(\textrm{\textsc{Round}}(x))] = \feo^{\xi}(x)$ \cite{calinescu2011maximizing}. 

\begin{theorem}[\citet{mokhtari2020stochastic}]
\label{thm:mokhtari}
Let $I^*$ be the maximizer of $\Feo^{\xi}$. \textsc{NMSCG} with pipage rounding, after $\bigO \left( p^{5/2}/\epsilon^3 \right)$ evaluations of $R$, achieves a solution $I$ such that 
$$\E\left[\Feo^{\xi}(I)\right] \geq \frac{1}{e} \Feo^{\xi}(I^*) - \epsilon.$$
\end{theorem}

The original result measures runtime in terms of the number of times we approximate the gradient in Equation~\ref{eqn:grad} with a single sample (in our case a single $G, I$ tuple). From Equations~\ref{eqn:f-R} and \ref{eqn:grad} we can see that the number of gradient approximations is a constant factor of the number of evaluations of $R$. Hence, we measure runtime in terms of number of evaluations of $R$. The bottleneck for evaluating $R$ is applying the Meek Rules, which can be computed in time polynomial in $p$ \cite{meek1995casual}. Note that the \textsc{NMSCG} subroutine can be modified to stabilize gradient updates by the approximation of a Hessian, in which case the same guarantee can be achieved in $\bigO \left( p^{3/2}/\epsilon^2 \right)$ \cite{hassani2019stochastic}. We use this version of NMSCG for the experiments. 
\begin{figure}[h!]
\begin{minipage}[t]{.48\linewidth}
\centering
\begin{algorithm}[H]
\caption{\scglong (\scg{})}
\label{alg:algo1}
\begin{algorithmic}
\STATE {\bfseries Input:} essential graph $\mathcal{G}$, constraints $C_{m, q}$, objective $\Feo$
 \STATE Init $\xi \leftarrow \emptyset$
 
 \WHILE{$\abs{\xi} \leq m$}
  \STATE $I \leftarrow \textrm{\textsc{Round}}(\textrm{\textsc{NMSCG}}(\Feo^{\xi}, q))$
  \STATE $\xi \leftarrow \xi \cup \{I\}$ 
 \ENDWHILE
 \STATE {\bfseries Output:} a set $\xi$ of interventions
\end{algorithmic}
\end{algorithm}
\end{minipage}
\quad
\begin{minipage}[t]{.48\linewidth}
\begin{algorithm}[H]
\caption{\ssglong (\ssg{})}
\label{alg:algo2}
\begin{algorithmic}
\STATE {\bfseries Input:} essential graph $\mathcal{G}$, constraints $C_{m, q}$, objective $\apFinf$
 \STATE Init $\mathcal{I} \leftarrow \emptyset$ 
 \STATE $\mathcal{S} \leftarrow \textrm{\textsc{SEPARATE}}(q, \mathcal{G})$
 \WHILE{$\abs{\xi} \leq m$}
  \STATE $I' \leftarrow \argmax_I \apFinf(\mathcal{I} \cup \{I \})$ 
  \STATE $\xi \leftarrow \xi \cup \{I'\}$ 
 \ENDWHILE
 \STATE {\bfseries Output:} a set $\xi$ of interventions
\end{algorithmic}
\end{algorithm}
\end{minipage}

\end{figure}

Our main result now follows from the fact that selecting interventions approximately greedily will lead to an approximation guarantee due to lemma~\ref{lem:F_mon_sub}.
 
\begin{theorem}
\label{thm:eo}
Let $\xi^* \in C_{N,b}$ be the maximizer of Objective \ref{obj:edgeorientobj}. \scg{} will, after $\bigO \left( m^4 p^{5/2}/\epsilon^3 \right)$ evaluations of $R$, achieve a solution $\xi$ such that 
$$\E[\Feo(\xi)] \geq \left( 1 - \frac{1}{e^{1/e}} \right) \Feo(\xi^*) - \epsilon.$$

\end{theorem}

We highlight this is a constant-factor guarantee with respect to the optimal batch of interventions. Our bound requires compute that is low order in $p$ with no dependence on $q$. Without even accounting for computing the possibly exponential sum in the objective, merely enumerating all possible interventions for fixed $q$ is $\mathcal{O}(p^q)$.

\citet{ghassami2018budgeted} give a similar constant-factor guarantee for batches of single-perturbation interventions ($q=1$). Their bound is within $1 - \frac{1}{e}$ of the \emph{optimal single-perturbation batch}. In the supplementary material, we show that the optimal multi-perturbation intervention can orient up to $q$ times more edges than the optimal single perturbation intervention. In Section~\ref{sec:experiments} we experimentally verify the value of multi-perturbation interventions by comparing \scg{} to the algorithm presented in \citet{ghassami2018budgeted}. If we allow for soft interventions, the $1-\frac{1}{e}$ guarantee can also be obtained for multi-perturbation interventions. See the supplementary material for details. 

\subsection{Optimizing the Mutual Information Objective}
\label{subsec:method:MI}

\looseness -1 We now consider an algorithm for maximizing Objective~\ref{obj:infABCDobj}. First, we note that computing the sum over $\mathcal{G}$ and the size of $\xi \cup \xi'$-essential graphs is computationally intractable. The computation of $\abs{\textrm{Ess}^{\xi \cup \xi'}(G)}$ makes \Finf{} a nested sum over DAGs. 
Like \citet{agrawal2019abcd}, we optimize a computationally tractable approximation to \Finf. First, we uniformly sample a multiset of DAGs $\tilde{\mathcal{G}}$ from $\mathcal{G}$ to construct our approximate objective
$$\apFinf(\xi) = - \frac{1}{\abs{\tilde{\mathcal{G}}}} \sum_{G\in \tilde{\mathcal{G}}} \log_2 \abs{\tilde{\textrm{Ess}}^{\xi \cup \xi'}(G)},$$
where $\tilde{\textrm{Ess}}^{\xi \cup \xi'}$ is the submultiset of $\tilde{\mathcal{G}}$ consisting of elements in $\textrm{Ess}^{\xi \cup \xi'}$. 
A submodularity result similar to lemma~\ref{lem:F_mon_sub} can also be proven for \apFinf. 

\begin{lemma}
\label{lem:U_mon_sub}
\apFinf{} is monotone submodular.
\end{lemma}

We first show that an approach similar to that used by \scg{} does not so easily give a near-optimal guarantee. Similarly to above, define $\apFinf^{\xi}(I) = \apFinf(\xi \cup \{I\})$ for fixed $\xi$.

\begin{proposition}
There exists $\mathcal{G}, \tilde{\mathcal{G}}$ such that $\apFinf^{\xi}$ is not submodular.
\end{proposition}

Hence we cannot use existing algorithms for submodular optimization to construct near-greedy interventions. We instead take a different approach. Suppose we can reduce $\mathcal{I}$ to some set of interventions $\mathcal{S}$ much smaller than $\mathcal{I}$, such that $\apFinf(\mathcal{S})$ has the maximum possible objective value. Due to lemma~\ref{lem:U_mon_sub}, we can obtain a guarantee by greedily selecting $m$ interventions from $\mathcal{S}$. A method for constructing the set $\mathcal{S}$ comes from {\em separating system} constructions. 

\begin{definition}
A $q$-sparse $\mathcal{G}$-separating system of size $N$ is a set of interventions $\mathcal{S} = \{S_1, S_2..., S_N \}$ such that $\abs{S_i} \leq q$ and for every undirected edge $(i, j) \in \mathcal{G}$ there is an element $S\in \mathcal{S}$ such that exactly one of $i, j$ is in $S$ \cite{shanmugam2015learning}. 
\end{definition}

A separating system of $\mathcal{G}$ completely identifies the true DAG, and hence obtains the maximum possible objective value $0$ without necessarily satisfying the constraints $C_{m,q}$. As an example, in Figure~\ref{fig:over}(b,c) we see 1 and 2--sparse separating systems respectively. 

We will make use of an algorithm $\textrm{\textsc{SEPARATE}}(q, \mathcal{G})$ which efficiently constructs a $q$-sparse separating system of $\mathcal{G}$. \citet{wegener1979separating} and \citet{shanmugam2015learning} give construction methods that are agnostic to the structure of $\mathcal{G}$ (it will identify any DAG with $p$ nodes). \citet{lindgren2018experimental} give a construction method that depends on the structure of $\mathcal{G}$. 

\looseness -1 Since the separating system obtains the maximum objective value, we can greedily select $m$ interventions from this set and obtain a lower bound on the objective value due to submodularity (lemma \ref{lem:U_mon_sub}). 

\begin{theorem}
\label{trm:ss}
\looseness -1 For $q\leq \lfloor{p/2 \rfloor}, m \leq \abs{\mathcal{S}}$, \ssg{}{} outputs $\xi \in C_{m,q}$ with objective value $$\apFinf(\xi) \geq (1-\frac{m}{\lceil p/q \rceil \lceil \log p \rceil}) \apFinf(\emptyset)$$

\looseness -1 in $\mathcal{O}(m \abs{\tilde{\mathcal{G}}} \frac{p}{q} \log p)$ evaluations of $R$, when using $\textrm{\textsc{SEPARATE}}$ as in \citet{shanmugam2015learning}.
\end{theorem}

Increasing $q$ does not necessarily increase $\apFinf(\xi)$. However, the bound we give becomes more favourable as $q$ increases because the upper bound on $\abs{\mathcal{S}}$ decreases. In practice, we run \ssg{} $\forall q' \leq q$ and pick the intervention set with the highest objective value on \apFinf. 

Note that \ssg{} can also be used with a similar guarantee for an analogous approximation of $\Feo$. 

\section{Experiments}
\label{sec:experiments}

\looseness -1 To evaluate our algorithms we consider three settings. Firstly, randomly generated DAGs, using infinite samples per intervention. Secondly, randomly generated DAGs with linear SEMs, using finite samples per intervention. Finally, subnetworks of GRN models, using infinite samples per intervention. Full details on all of our experiments can be found in the supplementary material. For code to reproduce the experiments, see \url{https://github.com/ssethz/multi-perturbation-ed}. 

\vspace{-2mm}
\looseness - 1 \paragraph{Infinite samples} We evaluate our algorithms using Objective~\ref{obj:edgeorientobj}. We consider selecting a batch of experiments where only the MEC is currently known. We vary the type of random graph and the constraint set $C_{m,q}$. The following methods are compared:

\begin{figure*}[t]
\begin{center}
\minipage{\columnwidth / 3}
\textbf{a}
\centering
\includegraphics[width=\columnwidth ]{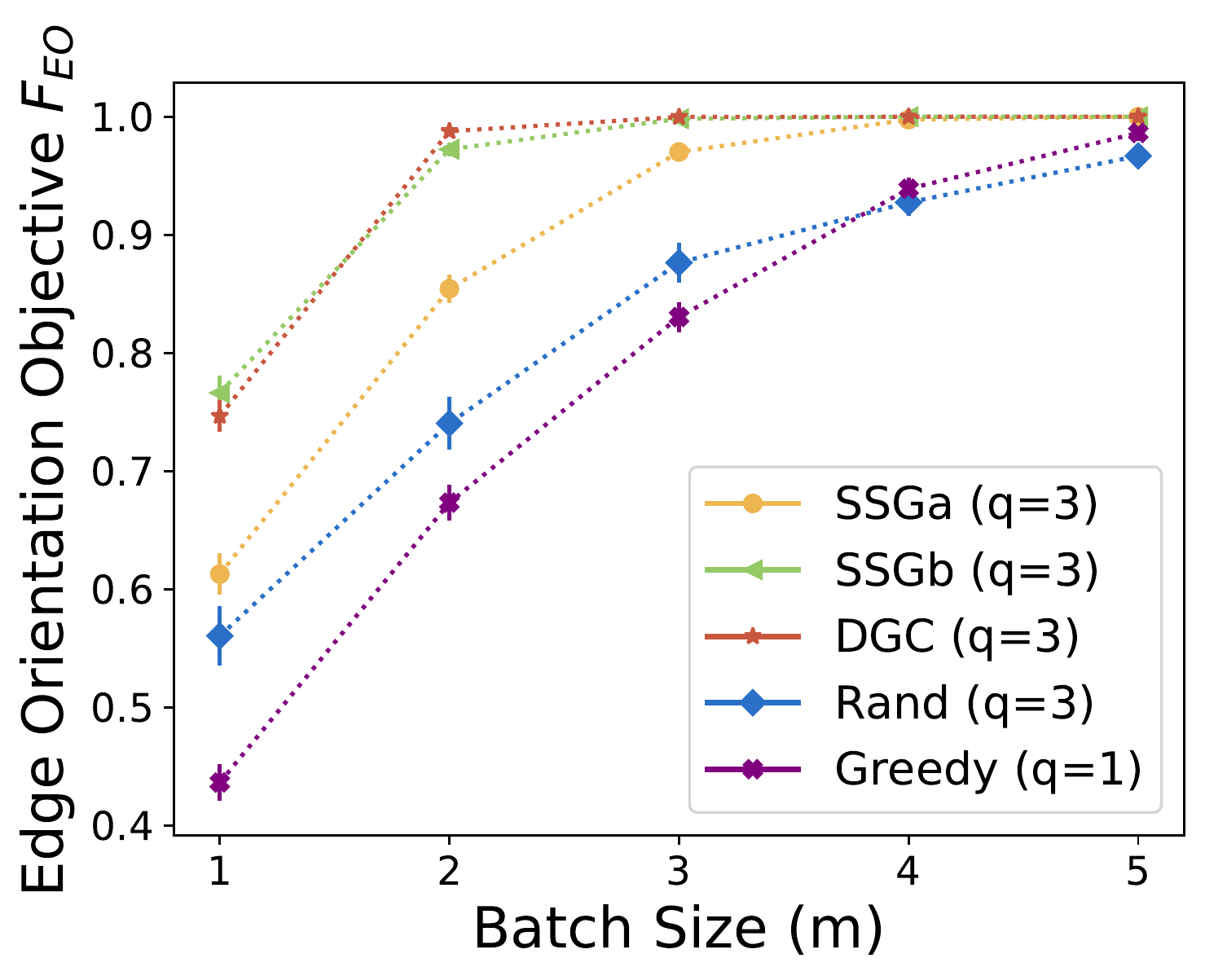}
\endminipage\hfill 
\minipage{\columnwidth / 3}
\textbf{b}
\centering
\includegraphics[width=\columnwidth ]{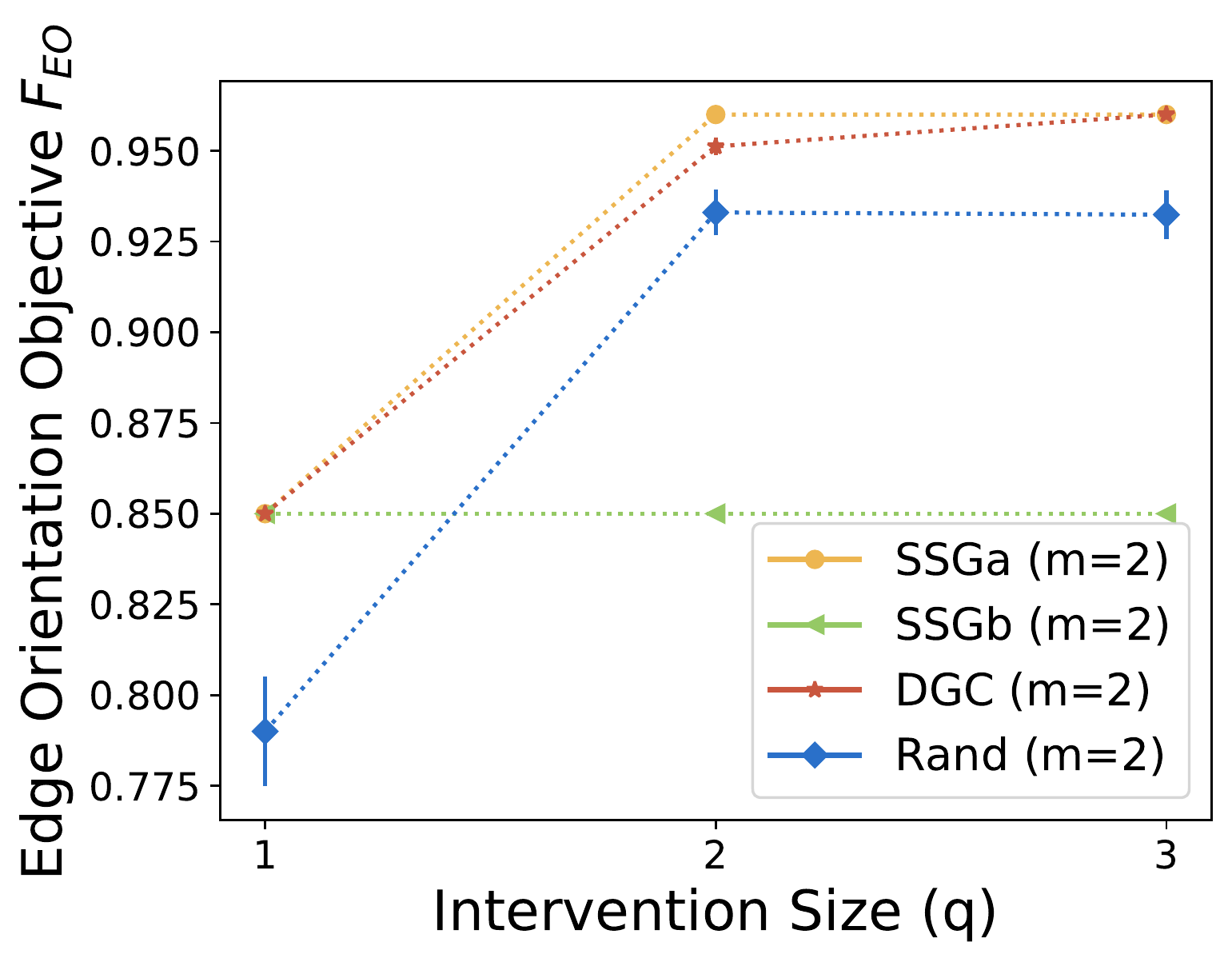}
\endminipage\hfill 
\minipage{\columnwidth / 3}
\textbf{c}
\centering
\includegraphics[width=\columnwidth ]{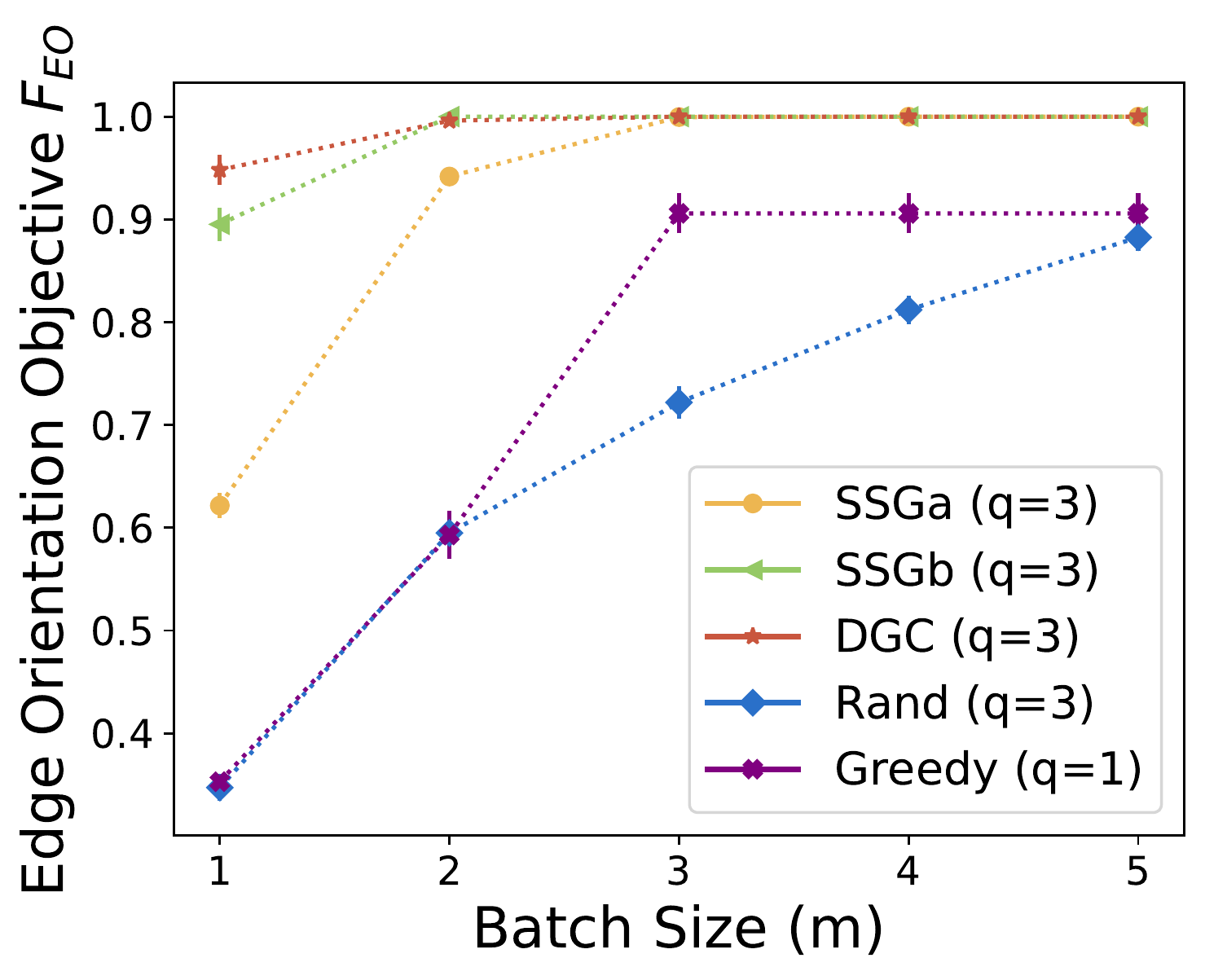}
\endminipage\hfill 
\minipage{\columnwidth/3}
\textbf{d}
\centering
\includegraphics[width=\columnwidth]{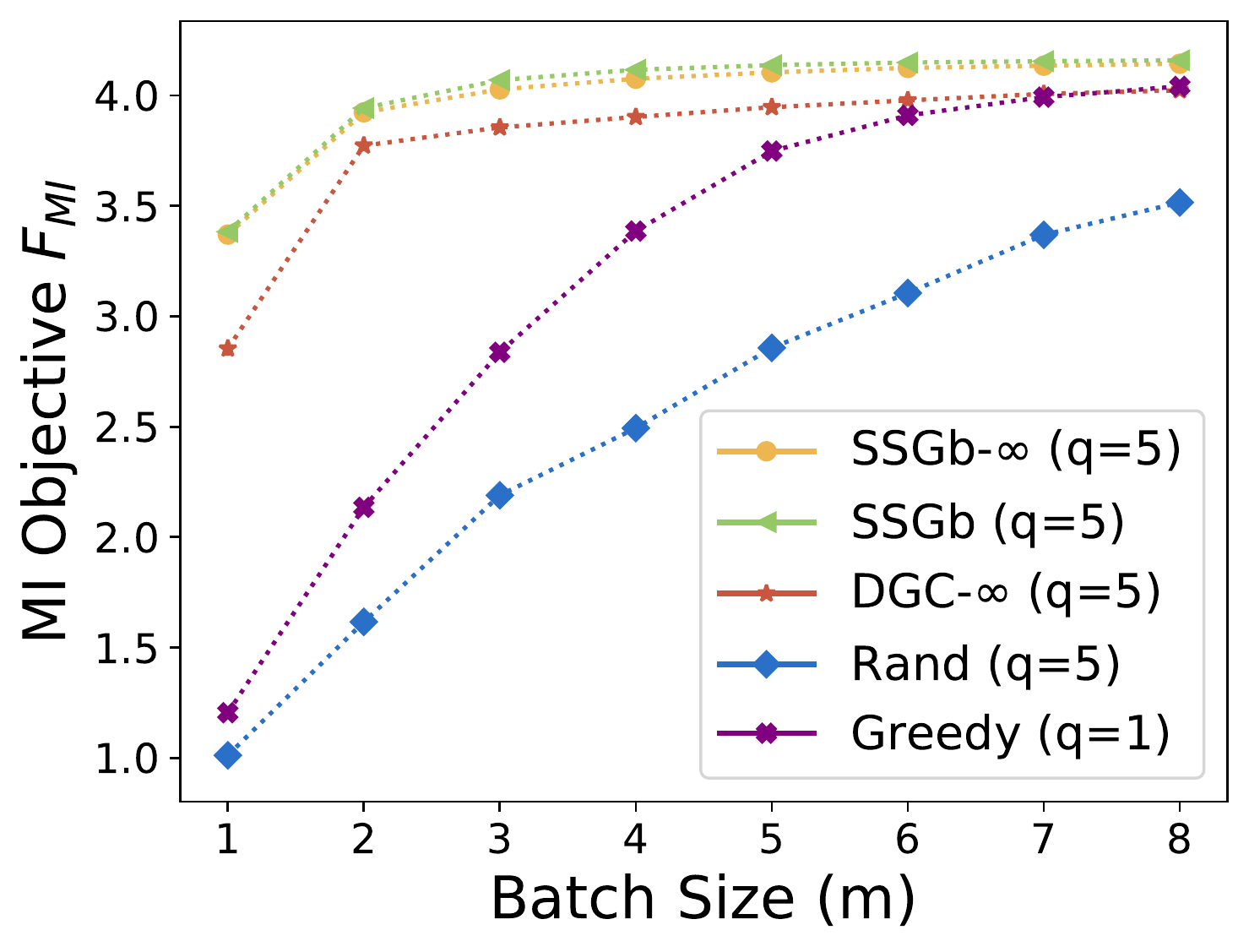}
\endminipage\hfill 
\minipage{\columnwidth/3}
\textbf{e}
\centering
\includegraphics[width=\columnwidth]{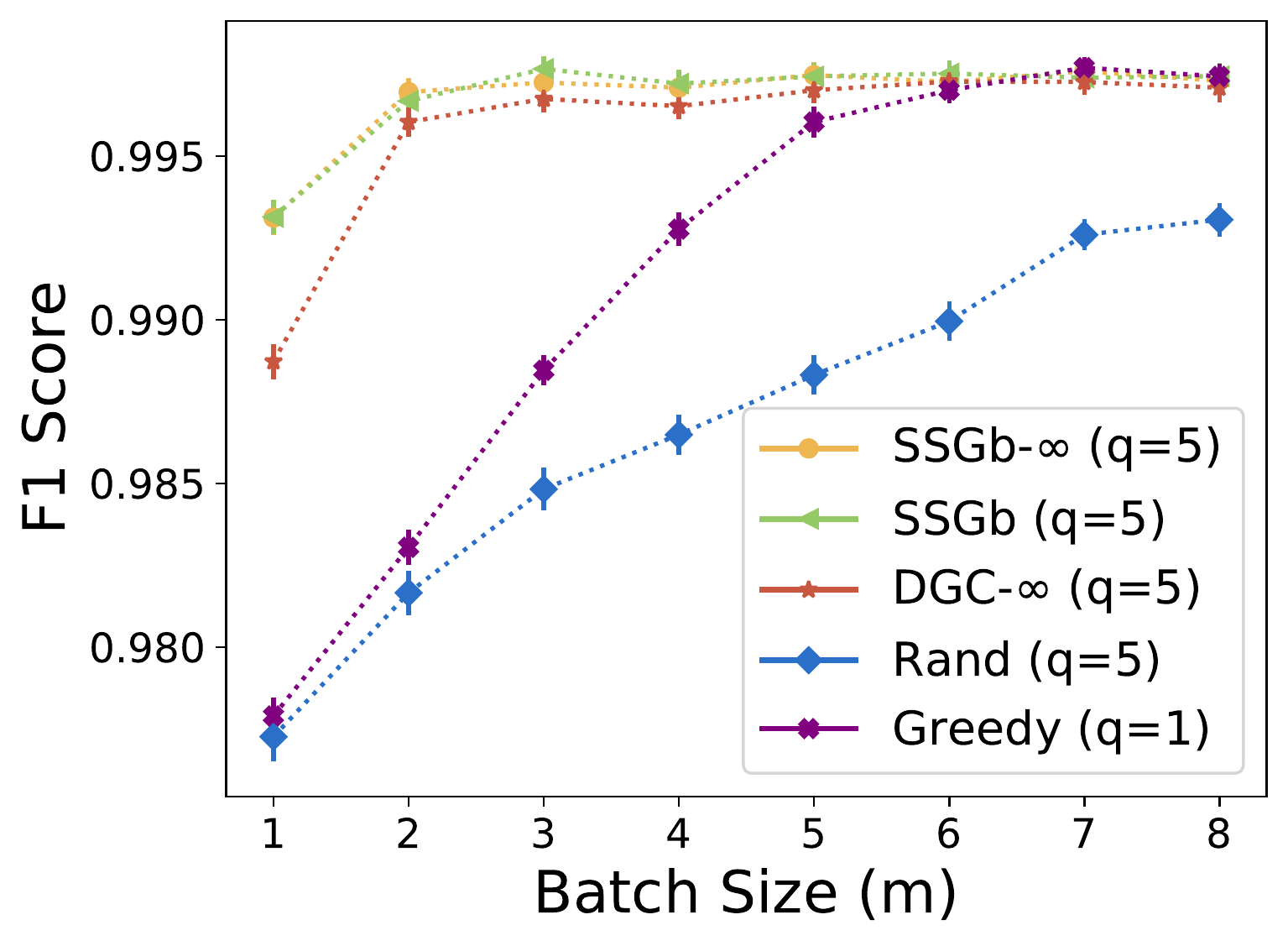}
\endminipage\hfill 
\minipage{\columnwidth/3}
\textbf{f}
\centering
\includegraphics[width=\columnwidth]{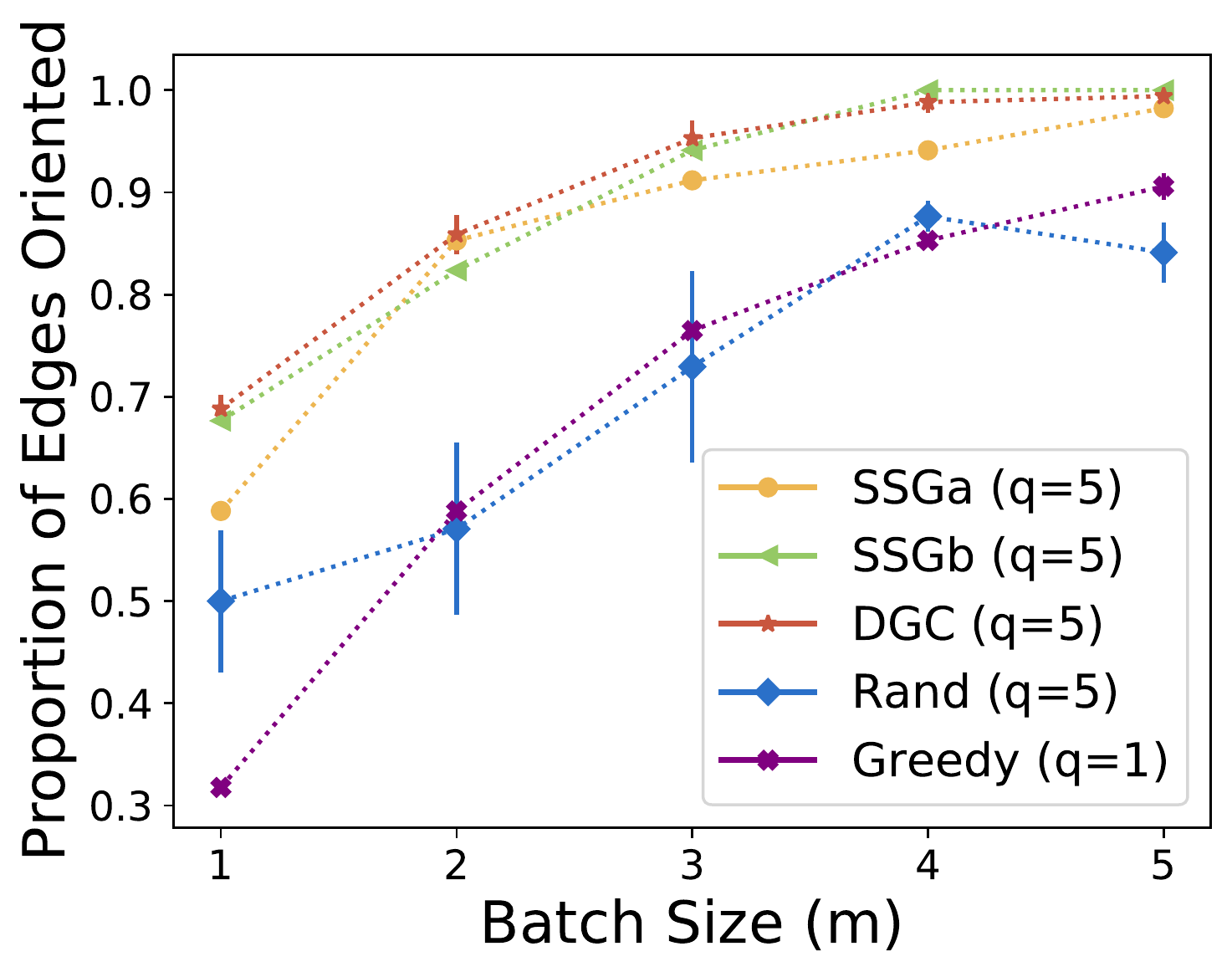}
\endminipage\hfill 

\caption{\looseness -1 a--c and f give infinite sample experiments and d--e give finite samples. \textbf{a)} Our algorithms ($q=3$) orient more edges than random interventions ($q=3$) and greedily chosen $q=1$ interventions for $p=40$, ER($0.1$) graphs. \textbf{b)} On a fully connected graph (p=5, m=2), \ssgb{} does not improve as $q$ increases. \textbf{c)} On a $p=20$ forest of 3 disconnected star graphs, \ssga{} cannot orient the full graph as quickly as our alternative approaches ($q=3$). \textbf{d)} For finite-samples, $p=40$, ER($0.1$) graphs, the proposed methods ($q=5$) achieve greater objective value than greedy $q=1$ interventions. Optimizing the finite sample objective directly yields slightly greater objective value than the infinite sample approximations. \textbf{e)} The F1 scores for predicting the presence of each edge correspond well with the objective in d).
\textbf{f)} For a $p=50$ yeast subnetwork from the DREAM3 challenge, our algorithms ($q=5$) orient more edges in the ground truth DAG than random or $q=1$ greedy algorithms with the same batch size.} 

\label{fig:inf}
\end{center}
\vskip -0.2in
\end{figure*}

\begin{itemize}[leftmargin=*]
    \item \looseness - 1 \rand: a baseline that for $m$ interventions, independently and uniformly at random selects $q$ nodes from those adjacent to at least one undirected edge;
    \item \greedy: greedily selects a single-perturbation intervention as in \citet{ghassami2018budgeted};
    \item \scg: our stochastic optimization approach;
    \item \looseness - 1 \ssga: our greedy approach using the graph agnostic separating systems of \citet{shanmugam2015learning};
    \item \looseness - 1 \ssgb: as above, using the graph-dependent separating system constructor of \citet{lindgren2018experimental}.
\end{itemize}

\looseness -1 Since there are infinite samples per intervention, the exact SEM used to generate data is not relevant. We plot the mean proportion of edges identified and error bars of 1 standard deviation over 100 repeats. Noise between repeats is due to randomness in the graph structure and in the algorithms themselves. 

In Figure \ref{fig:inf} a) we display the results for Erd{\"o}s-Reny{\'i} random graphs  with edge density $0.1$ (ER($0.1$)) and $40$ nodes. To prevent trivial graphs and large runtimes, we only consider graphs with MEC sizes in the range $[20, 200]$. The observations given here were also found for denser Erd{\"o}s-Reny{\'i} graphs and tree graphs, in addition to graphs with less nodes. 

\looseness -1 For all constraint values, all the algorithms improve greatly over \rand{} and \greedy. \ssgb{} outperforms \ssga{}, likely because the graph-sensitive separating system construction tends to return a groundset of more effective interventions. 

\ssgb{} achieves similar objective value to \scg{}. However, \scg{} behaves most robustly when the graph structure is chosen adversarially. For example, consider Figure \ref{fig:inf} b). Here we plot the proportion of identified edges on a $p=5$ fully connected graph. On this graph, the separating system construction of \citet{lindgren2018experimental} will always return the set of all single-node interventions. Therefore, its performance does not improve with $q$, whilst \scg{}'s does.  
An adversarial example for \ssga{} is constructed in Figure \ref{fig:inf} c): an MEC that consists of 3 disconnected star graphs with $7$, $7$ and $6$ nodes. In this case, \scg{} and \ssgb{} can orient most of the graph in a single intervention, whereas \ssga{} will likely not contain such an intervention in the separating system it constructs. 

\vspace{-2mm} 
\paragraph{Finite samples} \looseness -1 We use linear SEMs, with weights generated uniformly in $[-1, -0.25] \cup [0.25, 1]$. Measurement noise is given by the standard normal distribution. The underlying DAGs are generated in the same way as the infinite sample experiments. Before experiment selection, we obtain 800 observational samples of the system. 3 samples are obtained for each intervention selected by our algorithms. Each perturbation fixes the value of a node to $5$. We approximate Objective~\ref{obj:abcdobj} using the methods of \citet{agrawal2019abcd}. In particular, an initial distribution over DAGs is estimated by bootstrapping the observational data and using the techniques of \citet{yang2018characterizing} to infer DAGs. Each DAG in the distribution is weighted proportionally to the likelihood of the observational data given the DAG and the maximum likelihood estimate of the linear model weights. The posterior over DAGs after interventions is computed by re-weighting the existing set of DAGs based on the likelihood of the combined observational and interventional data. To ensure the distribution has support near the true DAG, we include all members of the true DAG's MEC in the initial distribution.  

With finite samples, our methods do not have guarantees but can be adapted into practical algorithms: 
\begin{itemize}[leftmargin=*]
    \item \looseness -1 \greedy: greedily optimize Objective~\ref{obj:abcdobj} with single perturbation interventions (\citet{agrawal2019abcd});
    \item \scg{}-$\infty$: optimizes Objective \ref{obj:edgeorientobj}, with the summation over DAGs being a weighted sum over the DAGs in the initial distribution;
    \item \looseness -1 \ssgb: optimizes \Fmi{}, greedily selecting from the separating system of \citet{lindgren2018experimental};
    \item \looseness -1 \ssgbinf: approximates the objective using Objective \ref{obj:infABCDobj} and optimizes with \ssg{}.  
\end{itemize}
For each algorithm, we record \Fmi{} of the selected interventions, over 200 repeats.  
In Figure~\ref{fig:inf} d) the objective values obtained by each algorithm are shown for varying batch size and $q=5$. \scg-$\infty$ performs worse than \ssgb{} and \ssgbinf{}, perhaps because it is optimizing \Feo{} which is not totally aligned with \Fmi{}. Between \ssgb{} and \ssgbinf{}, there was a small benefit to directly optimizing \Fmi{} as opposed to its infinite sample approximation \Finf{}. Accounting for finite samples may lead to greater improvements when there is heteroscedastic noise or a wider range of weights. 

\looseness -1 Performance on \Fmi{} corresponds closely with performance on a downstream task as shown in Figure~\ref{fig:inf} e). For each algorithm, we compute the posterior over DAGs given the selected interventions. Then, we independently predict the presence of each edge in the true DAG. Figure~\ref{fig:inf} e) plots the average F1 score for each algorithm. In finite samples, our approaches outperform both \rand{} and \greedy{}.

\vspace{-2mm}
\paragraph{DREAM 3 networks} We evaluate our algorithms under infinite samples using subgraphs of GRN models. From the  ``DREAM 3 In Silico Network” challenge \cite{marbach2009generating}, we use the 5 subgraphs with $p=50$ nodes. Here, we present the results for ``Yeast1", which is the graph requiring the most interventions for our methods to orient. Results for the other graphs are in the supplementary material. 

\looseness -1 We compare the same algorithms as considered in the other infinite sample experiments. In Figure~\ref{fig:inf} f) we record the proportion of unknown edges that were oriented by each algorithm. Our methods and \rand{} all have intervention sizes of 5. For each method we perform 5 repeats.
On Yeast1, our methods all perform similarly and outperform \rand{} and \greedy{}. This is found for the other DREAM 3 graphs too, except on one subgraph where \greedy{} performs similarly to our methods. 

\section{Conclusions}
\label{sec:discuss}
\looseness - 1 We presented near-optimal algorithms for causal structure learning through multi-perturbation interventions. Our results make novel use of submodularity properties and separating systems to search over a doubly exponential domain. Empirically, we demonstrated that these algorithms yield significant improvements over random interventions and state-of-the-art single-perturbation algorithms. These methods are particularly relevant in genomics applications, where causal graphs are large but multiple genes can be intervened upon simultaneously. 
\nocite{imoto2007analysis, karp1972reducibility, welsh1967upper, papadimitriou1991optimization, ghassami2019interventional}

\begin{ack}
This research was supported in part by the Swiss National Science Foundation, under NCCR Automation, grant agreement 51NF40 180545. Caroline Uhler was partially supported by NSF (DMS-1651995), ONR (N00014-17-1-2147 and N00014-18-1-2765), IBM, and a Simons Investigator Award.

Thank you to Raj Agrawal for a helpful discussion regarding experiments. 

Experiments were performed on the Leonhard cluster managed by the HPC team at ETH Z\"{u}rich

\end{ack}
\small
\printbibliography

\normalsize
\appendix
\renewcommand{\thesection}{\Alph{section}}
\renewcommand{\thesubsection}{A.\arabic{subsection}}

\newpage

\section*{Appendix}

\subsection{Potential Negative Societal Impacts}

We propose algorithms for experimental design when learning causal structures. The most obvious application is in designing scientific experiments to learn gene regulatory networks. It is possible that the ability to gather increasingly detailed information about gene networks could be used by malicious actors. However, to the authors' knowledge there are no examples of such use with existing related methods. On the other hand, there is longstanding scientific interest in learning such networks and potentially beneficial applications in drug design \cite{imoto2007analysis}. 

\subsection{The Power of Multi-Perturbation Interventions}
We can construct a more general example than the one given in Figure~\ref{fig:over}, to demonstrate that the optimal $q$-node intervention can learn up to $q$-times more edges than the optimal single node intervention. Consider a graph with MEC that is a forest of undirected star graphs, each with an equal number of nodes. The optimal single-node intervention can intervene on the center node in one of these stars, and entirely orient that star. The optimal $q$-node intervention can achieve this for $q$ of the stars. 

This example illustrates the upper bound on the number of additional edges that can be oriented by the optimal $q$-node intervention compared to the optimal single node intervention. It can be seen that this is a tight upper bound from the submodularity of $\Feo^{\xi}$ in lemma~2. 

\subsection{Deriving \Finf{}}
\citet{agrawal2019abcd} derive this infinite sample objective for the case of having gathered only infinite observational data before doing experiments. For the case of having infinite samples from observational data and some interventions, we follow a similar argument. 

For evaluating $U_{M.I}(y, \xi, D) = H(G \mid D) - H(G \mid D, y, \xi)$ where $H$ is the entropy, we only consider the second term. The first term does not depend on the interventions we do so is irrelevant for optimizing the objective. 

Given $D$ consisting of existing interventions $\xi'$ and observational data, the true DAG is already recovered up to it's $\xi'$-MEC. After obtaining infinite samples from each intervention in intervention set $\xi$, we recover the true DAG up to its $\xi' \cup \xi$-MEC. Therefore $H(G \mid D, y, \xi) = \log_2(\abs{Ess^{\xi \cup \xi'}(G)})$ when the true DAG is $G$. $Ess^{\xi \cup \xi'}(G)$ is the essential graph obtained after interventions $\xi$. Our prior distribution over DAGs is uniform over the $\xi'$-MEC of the true DAG. Averaging over these possibilities, the final objective is given by Objective~\ref{obj:infABCDobj}.

\subsection{Background on the Meek Rules}
\begin{figure*}[t]
\vskip 0.2in
\minipage{0.49\textwidth}
\centering
\begin{tikzpicture}[scale = 0.9, roundnode/.style={circle, draw=black!100, fill=black!100, thick, minimum size=2mm},
emptynode/.style={circle, draw=black!0, fill=black!0, thick, minimum size=2mm},
]

\node[roundnode] (d0) at (0, 7.5) {};
\node[roundnode] (d1) at (0, 6) {};
\node[roundnode] (d2) at (1.5, 6) {};
\node[roundnode] (d3) at (3.5, 7.5) {};
\node[roundnode] (d4) at (3.5, 6) {};
\node[roundnode] (d5) at (5, 6) {};

\node[emptynode]  (d6) at (2.5, 6.75) {$\implies$};
\node[emptynode]  (d7) at (2.5, 7.5) {R1};

\draw[->-, thick] (d0) -- (d1);
\draw[-, thick] (d1) -- (d2);
\draw[->-, thick] (d3) -- (d4);
\draw[->-, thick] (d4) -- (d5);
\end{tikzpicture}
\endminipage\hfill \vline
\minipage{0.49\textwidth}
\centering
\begin{tikzpicture}[scale = 0.9, roundnode/.style={circle, draw=black!100, fill=black!100, thick, minimum size=2mm},
emptynode/.style={circle, draw=black!0, fill=black!0, thick, minimum size=2mm},
]

\node[roundnode] (d0) at (0, 7.5) {};
\node[roundnode] (d1) at (0, 6) {};
\node[roundnode] (d2) at (1.5, 6) {};
\node[roundnode] (d3) at (3.5, 7.5) {};
\node[roundnode] (d4) at (3.5, 6) {};
\node[roundnode] (d5) at (5, 6) {};

\node[emptynode]  (d6) at (2.5, 6.75) {$\implies$};
\node[emptynode]  (d7) at (2.5, 7.5) {R2};

\draw[->-, thick] (d0) -- (d1);
\draw[->-, thick] (d1) -- (d2);
\draw[-, thick] (d2) -- (d0);
\draw[->-, thick] (d3) -- (d4);
\draw[->-, thick] (d4) -- (d5);
\draw[->-, thick] (d3) -- (d5);
\end{tikzpicture}
\endminipage\hfill 
\vskip 0.2in
\hrulefill
\vskip 0.2in
\minipage{0.49\textwidth}
\centering
\begin{tikzpicture}[scale = 0.9, roundnode/.style={circle, draw=black!100, fill=black!100, thick, minimum size=2mm},
emptynode/.style={circle, draw=black!0, fill=black!0, thick, minimum size=2mm},
]

\node[roundnode] (d0) at (0, 7.5) {};
\node[roundnode] (d1) at (0, 6) {};
\node[roundnode] (d2) at (1.5, 6) {};
\node[roundnode] (d3) at (1.5, 7.5) {};
\node[roundnode] (d4) at (3.5, 7.5) {};
\node[roundnode] (d5) at (3.5, 6) {};
\node[roundnode] (d6) at (5, 6) {};
\node[roundnode] (d7) at (5, 7.5) {};

\node[emptynode]  (d8) at (2.5, 6.75) {$\implies$};
\node[emptynode]  (d9) at (2.5, 7.5) {R3};

\draw[-, thick] (d0) -- (d1);
\draw[-, thick] (d1) -- (d2);
\draw[-, thick] (d0) -- (d3);
\draw[-, thick] (d2) -- (d0);
\draw[->-, thick] (d3) -- (d2);
\draw[->-, thick] (d1) -- (d2);

\draw[-, thick] (d4) -- (d5);
\draw[-, thick] (d5) -- (d6);
\draw[-, thick] (d4) -- (d7);
\draw[->-, thick] (d4) -- (d6);
\draw[->-, thick] (d7) -- (d6);
\draw[->-, thick] (d5) -- (d6);
\end{tikzpicture}
\endminipage\hfill \vline
\minipage{0.49\textwidth}
\centering
\begin{tikzpicture}[scale = 0.9, roundnode/.style={circle, draw=black!100, fill=black!100, thick, minimum size=2mm},
emptynode/.style={circle, draw=black!0, fill=black!0, thick, minimum size=2mm},
]

\node[roundnode] (d0) at (0, 7.5) {};
\node[roundnode] (d1) at (0, 6) {};
\node[roundnode] (d2) at (1.5, 6) {};
\node[roundnode] (d3) at (1.5, 7.5) {};
\node[roundnode] (d4) at (3.5, 7.5) {};
\node[roundnode] (d5) at (3.5, 6) {};
\node[roundnode] (d6) at (5, 6) {};
\node[roundnode] (d7) at (5, 7.5) {};

\node[emptynode]  (d8) at (2.5, 6.75) {$\implies$};

\node[emptynode]  (d9) at (2.5, 7.5) {R4};

\draw[->-, thick] (d0) -- (d1);
\draw[-, thick] (d1) -- (d2);
\draw[-, thick] (d0) -- (d3);
\draw[-, dotted, thick] (d1) -- (d3);
\draw[-, thick] (d3) -- (d2);
\draw[->-, thick] (d1) -- (d2);

\draw[->-, thick] (d4) -- (d5);
\draw[-, thick] (d5) -- (d6);
\draw[-, thick] (d4) -- (d7);
\draw[-, dotted, thick] (d5) -- (d7);
\draw[->-, thick] (d7) -- (d6);
\draw[->-, thick] (d5) -- (d6);
\end{tikzpicture}
\endminipage\hfill
\begin{center}
\caption{The 4 Meek rules. When a pattern on the left of the implication occurs, edges are oriented according to the pattern on the right of the implication. A dashed line means the direction of the edge may or may not be already identified. The name of each Meek rule is given above the implication sign. R0 refers to the step of orienting edges before Meek rules are applied.}
\label{fig:meek}
\end{center}
\vskip -0.2in
\end{figure*}
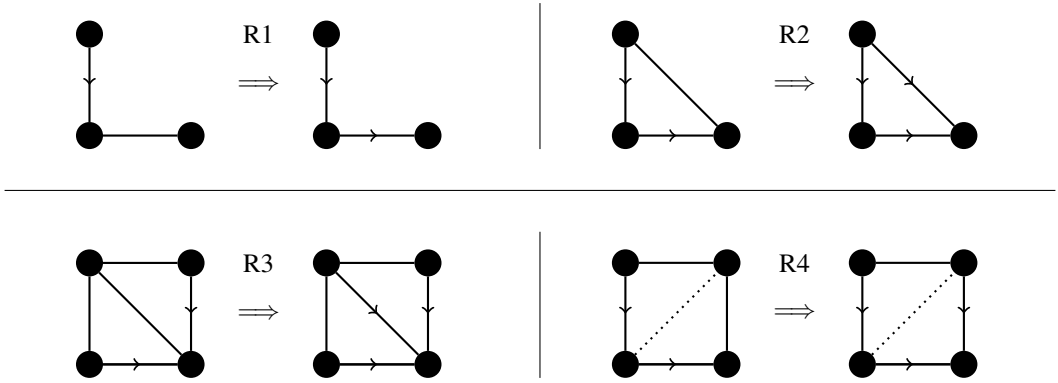

After performing interventions, the Meek rules can be used to orient additional edges. The rules are given in Figure~\ref{fig:meek}. The Meek rules are continually applied until none of the left side patterns appear in the partially directed graph.

\subsection{\Feo{} is consistent}

We use the same definition of {\emph{budgeted batch consistency}} introduced in \citet{agrawal2019abcd}. 

\begin{definition}[\citet{agrawal2019abcd}]
Assume our goal is to identify the true DAG $G^*$. Let us have constraints for the $b$th batch ($0 \leq b \leq B$) of experiments $C^b_{m,q}$ ($q \geq 1$). 
Objective $F$ is budgeted batch consistent if maximizing it in every batch implies

$$\Prob(G \mid D_B) \rightarrow \mathbbm{1}(G = G^*)$$
asymptotically as $m,B \rightarrow \infty$, where $D_B$ is the combined data obtained from all batches of experiments and the original dataset. 
\end{definition}

We show that our definition of \Feo{} satisfies budgeted batch consistency. Then we explain how the slight difference in the definition of \Feo{} given in \citet{agrawal2019abcd} leads to the authors concluding that it is not budgeted batch consistent. 

\begin{proposition}
\label{pro:consistent}
\Feo{} is budgeted batch consistent. 
\end{proposition}
\begin{proof}
Since each intervention offers an infinite number of samples, we can reason directly about the subsequent orienting of edges due to the obtained samples. Let the set of interventions after $b$ batches be $\xi_{b}$. We simply need to prove that as $k\rightarrow \infty$, $\xi_{b}$ will identify the orientation of every edge in the true DAG. Equivalently, we show that if $\xi_{b}$ has not oriented every edge (the current essential graph has size greater than $1$), $\xi_{b+1}$ will orient additional edges in the true DAG. There are only a finite number of edges to be oriented, so this means  $\xi_{\infty}$ will fully identify the true DAG. 

If $\xi_{b}$ does not orient every edge in the true DAG, Objective~\ref{obj:edgeorientobj} will have a maximum in the next batch of greater than $0$. Moreover, to obtain objective greater than $0$, the interventions selected in batch $b+1$  must orient edges not oriented by $\xi_{b}$. This is because after each batch, the objective is updated to be with reference to the essential graph of $G$ under interventions in $\xi_{b}$. To orient unidentified edges in any of the possible DAGs (those in the current $\xi_{b}$-MEC), we must orient at least one edge in the true DAG after obtaining samples. This can be seen from the first step of computing function $R$. One can also see that if an edge can be oriented, it can always be oriented by selecting a single unique intervention of size $1$ and thus the constraints $C_{m,q}$ can always be satisfied. 
\end{proof}

The key difference between this argument and the one given by \citet{agrawal2019abcd} is that they maintain a static objective function between batches, so the same set of interventions is selected every round.

\subsection{Proof of Lemma~1}

We'll work with the notation of $\xi_1 \subset \xi_2$ are sets of interventions, and $I$ is an intervention. We'll write $R(\xi, G) = M(A(\xi, G), G)$. Here, $A$ carries out the first step of orienting edges based on one of the nodes in that edge being intervened on (we refer to this as R0). $M$ carries out the Meek rules given the orientations of edges in $A$. Note that $R, M, A$ implicitly depend on the interventions carried out in previous batches $\xi'$, since this determines what edges might already be oriented in $G$ (whether we are orienting the MEC or some essential graph). 

We'll first show that function $R$ has a monotonicity-like property: adding an intervention only adds to the set of oriented edges. 
\begin{proposition}
\label{prop:mono}
Monotonicity-like property of $R$: $R(\xi_1, G) \subseteq R(\xi_2, G)$ for all $G$, $\xi_1 \subseteq \xi_2$. 
\end{proposition}
\begin{proof}
The same argument is given in \citet{ghassami2018budgeted}. By the definition of $A$, $A(\xi_1, G) \subseteq A(\xi_2, G)$. The Meek rules are sound and order-independent \cite{meek1995casual}, and therefore $M(A(\xi_1, G)) \subseteq M(A(\xi_2, G))$
\end{proof}
From this we can see that $\Feo$ is also monotonic. 

To prove lemma~\ref{lem:F_mon_sub}, we swill first prove some propositions regarding the marginal change in $R$ when adding a new intervention. 

\begin{proposition}
\label{prop:triangles}
Consider vertices $v_a$, $v_b$, $v_c$. Consider $\bar{G}$, the partially directed graph obtained after doing intervention set $\xi$ and then applying the Meek rules exhaustively. If $v_a\rightarrow v_b \in \bar{G}$, and $v_b-v_c \in \bar{G}$, then $v_a \rightarrow v_c \in \bar{G}$. 
\end{proposition}
\begin{proof}
If $v_a \rightarrow v_b$, and $v_b - v_c$, we must have that $v_a$ and $v_c$ are adjacent, else $v_b \rightarrow v_c$ by R1. We cannot have $v_c \rightarrow v_a$ since this would identify $v_b -v_c$ by R2. Hence we have either that $v_a-v_c$ or $v_a \rightarrow v_c$. 

Suppose for contradiction that, after applying all Meek rules, for some nodes $v_a,v_b,v_c$ we have $v_a \rightarrow v_b$, $v_b - v_c$ and $v_a-v_c$. We will gain a contradiction by an infinite descent. Any DAG can be associated with some permutation of its nodes that specifies a topological ordering, with the lowest ranked node being the root. Suppose that $B$ is the lowest ranked node in the topological ordering given by the true DAG (closest to the root) such that the supposed pattern holds. For all ways in which $v_a\rightarrow v_b$ could have been identified, we will show that either in fact $v_a\rightarrow v_c$ or find a vertex lower than $v_b$ in the topological ordering that fits into an identical pattern. 

Some cases are covered in \citet{meek1995casual} when proving a similar result (lemma 1). 

Suppose $v_a\rightarrow v_b$ is known by being a collider (identified before any interventions take place). This is handled in \citet{meek1995casual}. 

Suppose $v_a\rightarrow v_b$ is known by R1, R2, R3. These cases are all handled by \citet{meek1995casual}.

Suppose $v_a\rightarrow v_b$ is learnt directly by an intervention (rule R0). If it is identfied in this way, there must exist an intervention $I$ such that exactly one of $v_a, v_b \in I$. However, in either case regardless of whether $v_c \in I$ or not, we identify one of $v_a-v_c$ or $v_b-v_c$ by R0. Hence the pattern cannot occur if $v_a\rightarrow v_b$ is identified by R0.

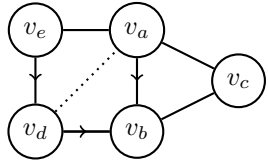
\begin{figure}[t]
\minipage{\textwidth}
\centering
\begin{tikzpicture}[scale = 0.9, roundnode/.style={circle, draw=black!100, fill=black!0, thick, minimum size=2mm},
emptynode/.style={circle, draw=black!0, fill=black!0, thick, minimum size=2mm},
]

\node[roundnode] (d0) at (0, 7.5) {$v_e$};
\node[roundnode] (d1) at (0, 6) {$v_d$};
\node[roundnode] (d2) at (1.5, 6) {$v_b$};
\node[roundnode] (d3) at (1.5, 7.5) {$v_a$};
\node[roundnode] (d4) at (3, 6.75) {$v_c$};

\draw[->-, thick] (d0) -- (d1);
\draw[-, thick] (d1) -- (d2);
\draw[-, thick] (d0) -- (d3);
\draw[-, dotted, thick] (d1) -- (d3);
\draw[->-, thick] (d3) -- (d2);
\draw[->-, thick] (d1) -- (d2);
\draw[-, thick] (d2) -- (d4);
\draw[-, thick] (d3) -- (d4);
\end{tikzpicture}
\endminipage\hfill

\caption{The pattern when an edge $v_a\rightarrow v_b$ is identified by R4 and $v_a-v_c$, $v_b-v_c$}.
\label{fig:R4_case_triangle}
\end{figure}

Suppose $v_a\rightarrow v_b$ is learnt directly by R4 as in Figure~\ref{fig:R4_case_triangle}. Consider the pattern given in Figure~\ref{fig:R4_case_triangle} where $v_e$ and $v_d$ are the other edges in the R4 pattern. Now if $v_a-v_d$ is undirected, then $AED$ is the same pattern as $v_a,v_b,v_c$ but with $v_d < v_b$ in the topological ordering (a contradiction). If $v_a\rightarrow v_d$, then $v_a,v_d,v_b$ gives the same setup as if we discovered $v_a\rightarrow v_b$ through R2. Similarly, if $v_d\rightarrow v_a$, then $v_e\rightarrow v_a$ by R2 and then we have the same setup as if we discovered $v_a\rightarrow v_b$ by R1. An alternative R4 pattern can also orient $v_a\rightarrow v_b$, however it involves node $v_c$ being part of the pattern. In this case, we must have oriented $v_c\rightarrow v_b$, a contradiction. 
\end{proof}

Proposition~\ref{prop:triangles} allows us to prove two propositions more directly related to our final result.

\begin{proposition}
\label{prop:R1}
$R(\xi_2 \cup \{I\}) \setminus R(\xi_2)  \subseteq R(\xi_1 \cup \{I\}) \setminus R(\xi_1)$
\end{proposition}
\begin{proof}
Here we'll take $R$ to also include all edges oriented before the interventions, since this doesn't change the outcome of the set difference operation above. We also drop $G$ from the notation since we work with a fixed true graph. This is just done out of convenience for the proof.

Take edge $e \in R(\xi_2 \cup \{I\})$ and $e \notin R(\xi_2)$. By the monotonicity-like property, $R(\xi_1) \subseteq R(\xi_2)$, so $e \notin R(\xi_1)$. Thus we just need to prove that for all such $e$, we have $e \in R(\xi_1 \cup \{I\})$. 

Assume for contradiction there is some nonempty set $E^{\dagger}$ of edges such that $\forall e \in E^{\dagger}$, $e \in R(\xi_2 \cup \{I\})$ and $e \notin R(\xi_2)$, but $e \notin R(\xi_1 \cup \{I\})$. We can specify an ordering over these edges. Order the edges (written $v_1\rightarrow v_2$) such that: the edges are in increasing order upon the position of $v_2$ in the topological ordering of graph $G$ (lower is closer to the root). Settle all ties by decreasing order on the position of $v_1$ in the topological ordering. We will now show that if there exists some $e$ that is the lowest ordered element in $E^{\dagger}$, we can always either create a contradiction or find some alternative edge in $E^{\dagger}$ with lower position in the ordering (itself a contradiction). 

If $e=v_1\rightarrow v_2$ is discovered in $R(\xi_2 \cup \{I\})$ by R0, we know $I$ must intervene on one of $v_1$, or $v_2$ but not the other. Hence $e\in R(\xi_1 \cup \{I\})$. 

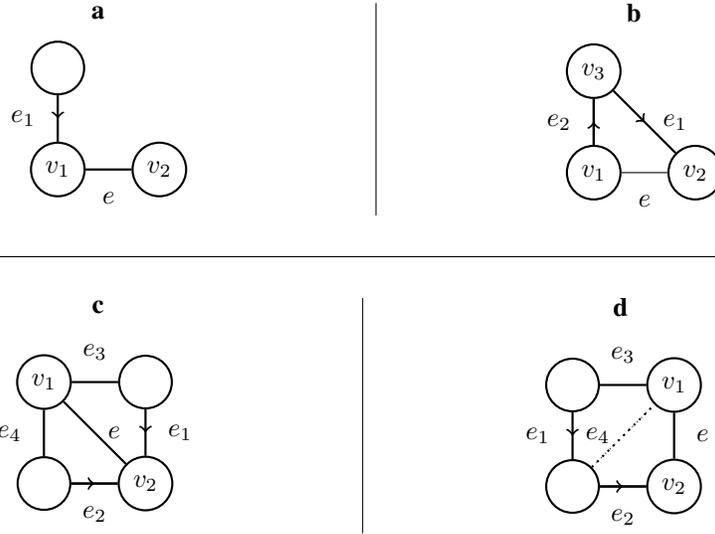
\begin{figure}[t]
\vskip 0.2in
\minipage{0.49\textwidth}
\centering
\textbf{a}
\vskip 0.1in
\begin{tikzpicture}[scale = 0.9, roundnode/.style={circle, draw=black!100, fill=black!0, thick, minimum size=7mm},
emptynode/.style={circle, draw=black!0, fill=black!0, thick, minimum size=2mm},
]

\node[roundnode] (d0) at (0, 7.5) {$ $};
\node[roundnode] (d1) at (0, 6) {$v_1$};
\node[roundnode] (d2) at (1.5, 6) {$v_2$};

\draw[->-, thick] (d0) -- (d1);
\draw[-, thick] (d1) -- (d2);

\draw (d0) -- (d1) node [midway, left=5pt, fill=white] {$e_1$};
\draw (d1) -- (d2) node [midway, below=5pt, fill=white] {$e$};
\end{tikzpicture}
\endminipage\hfill \vline
\minipage{0.49\textwidth}
\centering
\textbf{b}
\vskip 0.1in
\begin{tikzpicture}[scale = 0.9, roundnode/.style={circle, draw=black!100, fill=black!0, thick, minimum size=7mm},
emptynode/.style={circle, draw=black!0, fill=black!0, thick, minimum size=2mm},
]

\node[roundnode] (d0) at (0, 7.5) {$v_3$};
\node[roundnode] (d1) at (0, 6) {$v_1$};
\node[roundnode] (d2) at (1.5, 6) {$v_2$};

\draw[->-, thick] (d1) -- (d0);
\draw[->-, thick] (d0) -- (d2);
\draw[-, thick] (d1) -- (d0);

\draw (d0) -- (d1) node [midway, left=5pt, fill=white] {$e_2$};
\draw (d0) -- (d2) node [midway, above=3.5pt, right=3.5pt, fill=none] {$e_1$};
\draw (d1) -- (d2) node [midway, below=5pt, fill=white] {$e$};
\end{tikzpicture}
\endminipage\hfill 
\vskip 0.2in
\hrulefill
\vskip 0.2in
\minipage{0.49\textwidth}
\centering
\textbf{c}
\vskip 0.1in
\begin{tikzpicture}[scale = 0.9, roundnode/.style={circle, draw=black!100, fill=black!0, thick, minimum size=7mm},
emptynode/.style={circle, draw=black!0, fill=black!0, thick, minimum size=2mm},
]

\node[roundnode] (d0) at (0, 7.5) {$v_1$};
\node[roundnode] (d1) at (0, 6) {};
\node[roundnode] (d2) at (1.5, 6) {$v_2$};
\node[roundnode] (d3) at (1.5, 7.5) {};

\draw[-, thick] (d0) -- (d1);
\draw[-, thick] (d1) -- (d2);
\draw[-, thick] (d0) -- (d3);
\draw[-, thick] (d2) -- (d0);
\draw[->-, thick] (d3) -- (d2);
\draw[->-, thick] (d1) -- (d2);

\draw (d0) -- (d2) node [midway, above=3.5pt, right=1.5pt, fill=none] {$e$};
\draw (d0) -- (d1) node [midway, left=5pt, fill=white] {$e_4$};
\draw (d2) -- (d1) node [midway, below=5pt, fill=white] {$e_2$};
\draw (d0) -- (d3) node [midway, above=5pt, fill=white] {$e_3$};
\draw (d2) -- (d3) node [midway, right=5pt, fill=white] {$e_1$};

\end{tikzpicture}
\endminipage\hfill \vline
\minipage{0.49\textwidth}
\centering
\textbf{d}
\vskip 0.1in
\begin{tikzpicture}[scale = 0.9, roundnode/.style={circle, draw=black!100, fill=black!0, thick, minimum size=7mm},
emptynode/.style={circle, draw=black!0, fill=black!0, thick, minimum size=2mm},
]

\node[roundnode] (d0) at (0, 7.5) {};
\node[roundnode] (d1) at (0, 6) {};
\node[roundnode] (d2) at (1.5, 6) {$v_2$};
\node[roundnode] (d3) at (1.5, 7.5) {$v_1$};

\draw[->-, thick] (d0) -- (d1);
\draw[-, thick] (d1) -- (d2);
\draw[-, thick] (d0) -- (d3);
\draw[-, dotted, thick] (d1) -- (d3);
\draw[-, thick] (d3) -- (d2);
\draw[->-, thick] (d1) -- (d2);

\draw[dotted] (d1) -- (d3) node [midway, above=3.5pt, left=1.5pt, fill=none] {$e_4$};
\draw (d0) -- (d1) node [midway, left=5pt, fill=white] {$e_1$};
\draw (d2) -- (d1) node [midway, below=5pt, fill=white] {$e_2$};
\draw (d0) -- (d3) node [midway, above=5pt, fill=white] {$e_3$};
\draw (d2) -- (d3) node [midway, right=5pt, fill=white] {$e$};

\end{tikzpicture}
\endminipage\hfill
\begin{centering}
\caption{The Meek rules with labels on edges and nodes. \textbf{a)} R1, \textbf{b)} R2, \textbf{c)} R3, \textbf{d)} R4.}
\end{centering}
\label{fig:meek_hand}
\end{figure}

The diagram in Figure~\ref{fig:meek_hand} is given for following the other cases. Notation refers to the names of nodes and edges given on the diagrams. 

If $e=v_1\rightarrow v_2$ is discovered in $R(\xi_2 \cup \{I\})$ by R1, we must have $e_1 \in R(\xi_2 \cup \{I\})$. We must also have that $e_1 \notin  R(\xi_2)$ else $e \in R(\xi_2)$ by R1. The same must be true of $e'$ not being in  $R(\xi_1 \cup \{I\})$, so $e_1 \in E^{\dagger}$. But $e_1$ will have lower ordering since $v_1$ is below $v_2$ in the topological ordering. 

If $e=v_1\rightarrow v_2$ is discovered in $R(\xi_2 \cup \{I\})$ by R2, we must have some pair $e_1, e_2 \in R(\xi_2 \cup \{I\})$ such that $e_1 = v_3 \rightarrow v_2$ and $e_2 = v_1 \rightarrow v_3$. One of these edges cannot be in $R(\xi_2)$. In fact, neither can be in $R(\xi_2)$ due to proposition~\ref{prop:triangles}. This is because assuming only one is identified, based on proposition~\ref{prop:triangles}, another edge is either identified which leads to orienting $e$ or another edge is incorrectly oriented in the true DAG. The same holds for $R(\xi_1 \cup \{I\})$, but then we've found an edge $e_2 \in E^{\dagger}$ that is lower in the ordering than $e$.

If $e=v_1\rightarrow v_2$ is discovered in $R(\xi_2 \cup \{I\})$ by R3, we must have $e_1, e_2 \in R(\xi_2 \cup \{I\})$ and $e_3, e_4 \in G$ but not necessarily in $R(\xi_2 \cup \{I\})$.  $e_1$ and $e_2$ form an unshielded collider and are identified before intervening, so $e \in R(\xi_1 \cup \{I\})$ by R3. 

If $e=v_1\rightarrow v_2$ is discovered in $R(\xi_2 \cup \{I\})$ by R4, we must have $e_1, e_2 \in R(\xi_2 \cup \{I\})$ and $e_3, e_4 \in G$ but not necessarily in $R(\xi_2 \cup \{I\})$. At least one of $e_1, e_2 \notin R(\xi_2)$. Suppose only $e_1$ is in, then $e_2$ is in by $R1$. Therefore $e_1$ is not in $R(\xi_2)$ or $R(\xi_1 \cup \{I\})$, but this is lower in the ordering than $e$.
\end{proof}

\begin{proposition}
\label{prop:R2}
$R(\xi_1) \setminus R(\xi_1 \cup \{I\})  \subseteq R(\xi_2) \setminus R(\xi_2 \cup \{I\})$
\end{proposition}
\begin{proof}
Follows by monotonicity of $R$, both sides are the empty set. 
\end{proof}

We can rewrite $F_{EO} = \sum_{G\in \mathcal{G}} g(R(\xi, G))$ where $g$ is the weighted coverage function. $g$ is a monotonic function of the set of edges in $R$. 

\begin{equation*}
\begin{split}
g(R(\xi_1 \cup \{I\})) - g(R(\xi_1)) &\overset{\mathrm{(i)}}{=}
g(R(\xi_1 \cup \{I\}) \setminus R(\xi_1)) \\
&\;\;- g(R(\xi_1) \setminus R(\xi_1 \cup \{I\})) \\
&\overset{\mathrm{(ii)}}{\geq} g(R(\xi_2 \cup \{I\}) \setminus R(\xi_2)) \\
&\;\;- g(R(\xi_2) \setminus R(\xi_2 \cup \{I\})) \\
&= g(R(\xi_2\cup \{I\})) - g(R(\xi_2)).
\end{split}
\end{equation*}
Step (i) is a property of the weighted coverage function. Step (ii) comes from propositions \ref{prop:R1} and \ref{prop:R2} and the monotonicity of $g$. This shows that definition~2 holds for $g(R(\xi, G))$ as a function of $\xi$ for all $G$. Since the sum of submodular functions is submodular, this implies lemma~\ref{lem:F_mon_sub}. 

\subsection{Proof of Lemma~2}
We can see that a monotonicity property like proposition~\ref{prop:mono} does not hold in this case. The intervention $[p]$, for example, orients no edges. Nevertheless, we follow the same approach to prove submodularity of $\Feo^{\xi}$.

For this we consider $I_1 \subset I_2$ and consider adding node $v \notin I_2$ to these interventions. 

For notational simplicity, $R(I, G)$ will now denote all edges oriented after intervention set $\xi$ (fixed) and intervention $I$ on true graph $G$. We will drop $G$ from the notation in cases where the graph is fixed. 

\begin{proposition}
\label{prop:R_inner1}
$R(I_2 \cup \{v\}) \setminus R(I_2)  \subseteq R(I_1 \cup \{v\}) \setminus R(I_1)$
\end{proposition}
\begin{proof}
We want to show two things. First we want to show that if $e \in R(I_2 \cup \{v\})$ and $e \notin R(I_2)$, then $e \in R(I_1 \cup \{v\})$. This is shown with an identical technique to the one in proposition~\ref{prop:R1}. This is because the Meek rules are the same in both cases. The only difference is the case when $e=v_1 \rightarrow v_2$ is discovered by R0. In this case, $v$ must be either $v_1$ or $v_2$ and neither of $v_1$ or $v_2$ can be in $I_2$. However therefore neither are in $I_1$ and hence $e \in R(I_1 \cup \{v\})$.

The second thing we need to show is that if $e \in R(I_2 \cup \{v\})$ and $e \notin R(I_2)$, then $e \notin R(I_1)$. Suppose for contradiction that there exists some such $e \in R(I_1)$. We'll proceed in two steps. First we'll show that in order to avoid a contradiction, we must have that $I_2$ intervenes on both vertices in $e$. Second we'll show that if we intervene on both vertices in $e$ for $I_2$, we cannot have that $e \in R(I_2\cup \{v\})$ and $e\notin R(I_2)$.

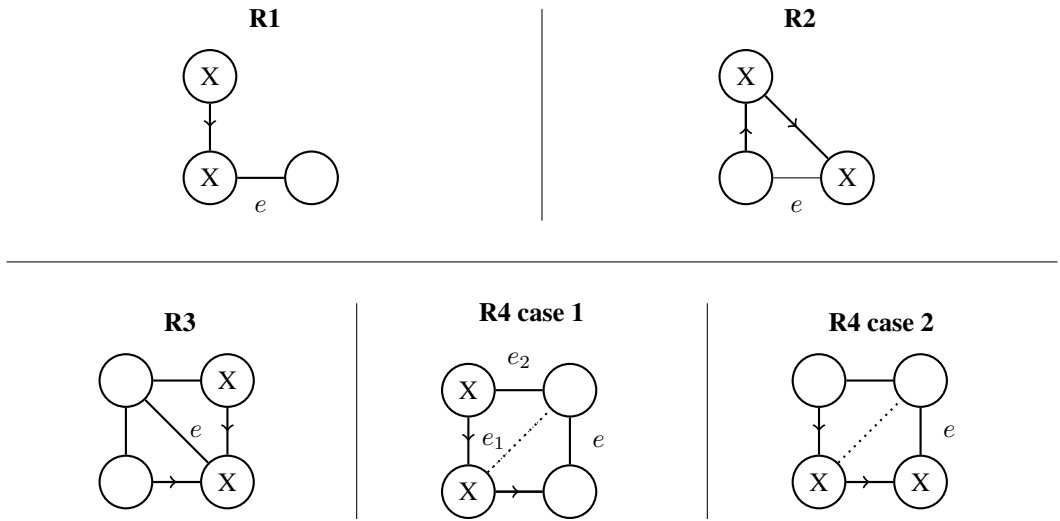
\begin{figure}[t]
\vskip 0.2in
\minipage{0.49\textwidth}
\centering
\textbf{R1}
\vskip 0.1in
\begin{tikzpicture}[scale = 0.9, roundnode/.style={circle, draw=black!100, fill=black!0, thick, minimum size=7mm},
emptynode/.style={circle, draw=black!0, fill=black!0, thick, minimum size=2mm},
]

\node[roundnode] (d0) at (0, 7.5) {X};
\node[roundnode] (d1) at (0, 6) {X};
\node[roundnode] (d2) at (1.5, 6) { };

\draw[->-, thick] (d0) -- (d1);
\draw[-, thick] (d1) -- (d2);

\draw (d1) -- (d2) node [midway, below=5pt, fill=white] {$e$};
\end{tikzpicture}
\endminipage\hfill \vline
\minipage{0.49\textwidth}
\centering
\textbf{R2}
\vskip 0.1in
\begin{tikzpicture}[scale = 0.9, roundnode/.style={circle, draw=black!100, fill=black!0, thick, minimum size=7mm},
emptynode/.style={circle, draw=black!0, fill=black!0, thick, minimum size=2mm},
]

\node[roundnode] (d0) at (0, 7.5) {X};
\node[roundnode] (d1) at (0, 6) {};
\node[roundnode] (d2) at (1.5, 6) {X};

\draw[->-, thick] (d1) -- (d0);
\draw[->-, thick] (d0) -- (d2);
\draw[-, thick] (d1) -- (d0);

\draw (d1) -- (d2) node [midway, below=5pt, fill=white] {$e$};
\end{tikzpicture}
\endminipage\hfill 
\vskip 0.2in
\hrulefill
\vskip 0.2in
\minipage{0.33\textwidth}
\centering
\textbf{R3}
\vskip 0.1in
\begin{tikzpicture}[scale = 0.9, roundnode/.style={circle, draw=black!100, fill=black!0, thick, minimum size=7mm},
emptynode/.style={circle, draw=black!0, fill=black!0, thick, minimum size=2mm},
]

\node[roundnode] (d0) at (0, 7.5) {};
\node[roundnode] (d1) at (0, 6) {};
\node[roundnode] (d2) at (1.5, 6) {X};
\node[roundnode] (d3) at (1.5, 7.5) {X};

\draw[-, thick] (d0) -- (d1);
\draw[-, thick] (d1) -- (d2);
\draw[-, thick] (d0) -- (d3);
\draw[-, thick] (d2) -- (d0);
\draw[->-, thick] (d3) -- (d2);
\draw[->-, thick] (d1) -- (d2);

\draw (d0) -- (d2) node [midway, above=3.5pt, right=1.5pt, fill=none] {$e$};

\end{tikzpicture}
\endminipage\hfill \vline
\minipage{0.33\textwidth}
\centering
\textbf{R4 case 1}
\vskip 0.1in
\begin{tikzpicture}[scale = 0.9, roundnode/.style={circle, draw=black!100, fill=black!0, thick, minimum size=7mm},
emptynode/.style={circle, draw=black!0, fill=black!0, thick, minimum size=2mm},
]

\node[roundnode] (d0) at (0, 7.5) {X};
\node[roundnode] (d1) at (0, 6) {X};
\node[roundnode] (d2) at (1.5, 6) {};
\node[roundnode] (d3) at (1.5, 7.5) {};

\draw[->-, thick] (d0) -- (d1);
\draw[-, thick] (d1) -- (d2);
\draw[-, thick] (d0) -- (d3);
\draw[-, dotted, thick] (d1) -- (d3);
\draw[-, thick] (d3) -- (d2);
\draw[->-, thick] (d1) -- (d2);

\draw[dotted] (d1) -- (d3) node [midway, above=3.5pt, left=1.5pt, fill=none] {$e_1$};
\draw (d3) -- (d0) node [midway, above=5pt, fill=white] {$e_2$};
\draw (d2) -- (d3) node [midway, right=5pt, fill=white] {$e$};
\end{tikzpicture}
\endminipage\hfill \vline
\minipage{0.33\textwidth}
\centering
\textbf{R4 case 2}
\vskip 0.1in
\begin{tikzpicture}[scale = 0.9, roundnode/.style={circle, draw=black!100, fill=black!0, thick, minimum size=7mm},
emptynode/.style={circle, draw=black!0, fill=black!0, thick, minimum size=2mm},
]

\node[roundnode] (d0) at (0, 7.5) {};
\node[roundnode] (d1) at (0, 6) {X};
\node[roundnode] (d2) at (1.5, 6) {X};
\node[roundnode] (d3) at (1.5, 7.5) {};

\draw[->-, thick] (d0) -- (d1);
\draw[-, thick] (d1) -- (d2);
\draw[-, thick] (d0) -- (d3);
\draw[-, dotted, thick] (d1) -- (d3);
\draw[-, thick] (d3) -- (d2);
\draw[->-, thick] (d1) -- (d2);

\draw (d2) -- (d3) node [midway, right=5pt, fill=white] {$e$};

\end{tikzpicture}
\endminipage\hfill
\caption{The pattern of each Meek rule as in the first part of proposition~\ref{prop:R_inner1}. An $X$ denotes that an intervention occured at that node. In our tree representation, if a child edge has both vertices intervened on then the parent is identified anyway, unless both vertices of the parent are intervened on.}
\label{fig:orienting_up_tree}
\end{figure}

We can represent the identification of edge $e$ in $R(I_1)$ by a directed tree diagram. The root in the diagram is $e$, and the children of each node in the diagram are the directed edges involved in the Meek rule that identifies the parent edge. Each node can have either one or two children (since each Meek Rule depends on up to two specific edges being directed). Leaf nodes must have been identified by R0. Since $e \in R(I_1)$ and $e \notin R(I_2)$, there must be leaf nodes in the diagram that are not identified by R0 using intervention $I_2$. Since $I_1 \subset I_2$, this means that for these leaf edges, $I_2$ must contain both nodes. However we now show that, for the structure of all Meek rules, in our tree diagram if a child edge has both vertices intervened on then the parent is identified anyway, unless both vertices of the parent are intervened on. This is shown pictorially in Figure~\ref{fig:orienting_up_tree}. R4 case 1 requires some extra explanation. $e_1$ is identified and must be oriented towards the node in $e$ to prevent identifying $e$ by R2. Hence $e_2$ must direct into $e$ to prevent a cycle. Then $e$ is learnt by R1. 
Given all this, we can see that to prevent identification of $e$ by $I_2$, there must be a path from a leaf edge to the root $e$ where all the edges in the path have both vertices in the edge intervened on in $I_2$. Hence $e$ has both vertices contained in $I_2$. 

\begin{figure}[t]
\vskip 0.2in
\minipage{0.33\textwidth}
\centering
\textbf{R1}
\vskip 0.1in
\begin{tikzpicture}[scale = 0.9, roundnode/.style={circle, draw=black!100, fill=black!0, thick, minimum size=7mm},
emptynode/.style={circle, draw=black!0, fill=black!0, thick, minimum size=2mm},
]

\node[roundnode] (d0) at (0, 7.5) { };
\node[roundnode] (d1) at (0, 6) {X};
\node[roundnode] (d2) at (1.5, 6) {X};

\draw[->-, thick] (d0) -- (d1);
\draw[-, thick] (d1) -- (d2);

\draw (d1) -- (d2) node [midway, below=5pt, fill=white] {$e$};
\end{tikzpicture}
\endminipage\hfill \vline
\minipage{0.33\textwidth}
\centering
\textbf{R2}
\vskip 0.1in
\begin{tikzpicture}[scale = 0.9, roundnode/.style={circle, draw=black!100, fill=black!0, thick, minimum size=7mm},
emptynode/.style={circle, draw=black!0, fill=black!0, thick, minimum size=2mm},
]

\node[roundnode] (d0) at (0, 7.5) {};
\node[roundnode] (d1) at (0, 6) {X};
\node[roundnode] (d2) at (1.5, 6) {X};

\draw[->-, thick] (d1) -- (d0);
\draw[->-, thick] (d0) -- (d2);
\draw[-, thick] (d1) -- (d0);

\draw (d1) -- (d2) node [midway, below=5pt, fill=white] {$e$};
\end{tikzpicture}
\endminipage\hfill \vline
\minipage{0.33\textwidth}
\centering
\textbf{R3}
\vskip 0.1in
\begin{tikzpicture}[scale = 0.9, roundnode/.style={circle, draw=black!100, fill=black!0, thick, minimum size=7mm},
emptynode/.style={circle, draw=black!0, fill=black!0, thick, minimum size=2mm},
]

\node[roundnode] (d0) at (0, 7.5) {X};
\node[roundnode] (d1) at (0, 6) {};
\node[roundnode] (d2) at (1.5, 6) {X};
\node[roundnode] (d3) at (1.5, 7.5) {};

\draw[-, thick] (d0) -- (d1);
\draw[-, thick] (d1) -- (d2);
\draw[-, thick] (d0) -- (d3);
\draw[-, thick] (d2) -- (d0);
\draw[->-, thick] (d3) -- (d2);
\draw[->-, thick] (d1) -- (d2);

\draw (d0) -- (d2) node [midway, above=3.5pt, right=1.5pt, fill=none] {$e$};

\end{tikzpicture}
\endminipage\hfill 
\vskip 0.2in
\hrulefill
\vskip 0.2in
\minipage{0.33\textwidth}
\centering
\textbf{R4 case 1}
\vskip 0.1in
\begin{tikzpicture}[scale = 0.9, roundnode/.style={circle, draw=black!100, fill=black!0, thick, minimum size=7mm},
emptynode/.style={circle, draw=black!0, fill=black!0, thick, minimum size=2mm},
]

\node[roundnode] (d0) at (0, 7.5) {};
\node[roundnode] (d1) at (0, 6) {};
\node[roundnode] (d2) at (1.5, 6) {X};
\node[roundnode] (d3) at (1.5, 7.5) {X};

\draw[->-, thick] (d0) -- (d1);
\draw[-, thick] (d1) -- (d2);
\draw[-, thick] (d0) -- (d3);
\draw[-, dotted, thick] (d1) -- (d3);
\draw[-, thick] (d3) -- (d2);
\draw[->-, thick] (d1) -- (d2);

\draw[dotted] (d1) -- (d3) node [midway, above=3.5pt, left=1.5pt, fill=none] {$e_1$};
\draw (d3) -- (d0) node [midway, above=5pt, fill=white] {$e_2$};
\draw (d2) -- (d3) node [midway, right=5pt, fill=white] {$e$};
\draw (d0) -- (d1) node [midway, left=5pt, fill=white] {$e_3$};
\end{tikzpicture}
\endminipage\hfill \vline
\minipage{0.33\textwidth}
\centering
\textbf{R4 case 2}
\vskip 0.1in
\begin{tikzpicture}[scale = 0.9, roundnode/.style={circle, draw=black!100, fill=black!0, thick, minimum size=7mm},
emptynode/.style={circle, draw=black!0, fill=black!0, thick, minimum size=2mm},
]

\node[roundnode] (d0) at (0, 7.5) {X};
\node[roundnode] (d1) at (0, 6) {};
\node[roundnode] (d2) at (1.5, 6) {X};
\node[roundnode] (d3) at (1.5, 7.5) {X};

\draw[->-, thick] (d0) -- (d1);
\draw[-, thick] (d1) -- (d2);
\draw[-, thick] (d0) -- (d3);
\draw[-, dotted, thick] (d1) -- (d3);
\draw[-, thick] (d3) -- (d2);
\draw[->-, thick] (d1) -- (d2);

\draw (d2) -- (d3) node [midway, right=5pt, fill=white] {$e$};
\draw (d0) -- (d1) node [midway, left=5pt, fill=white] {$e_3$};
\draw (d2) -- (d1) node [midway, below=5pt, fill=white] {$e_4$};
\end{tikzpicture}
\endminipage\hfill \vline
\minipage{0.33\textwidth}
\centering
\textbf{R4 case 3}
\vskip 0.1in
\begin{tikzpicture}[scale = 0.9, roundnode/.style={circle, draw=black!100, fill=black!0, thick, minimum size=7mm},
emptynode/.style={circle, draw=black!0, fill=black!0, thick, minimum size=2mm},
]

\node[roundnode] (d0) at (0, 7.5) {};
\node[roundnode] (d1) at (0, 6) {X};
\node[roundnode] (d2) at (1.5, 6) {X};
\node[roundnode] (d3) at (1.5, 7.5) {X};

\draw[->-, thick] (d0) -- (d1);
\draw[-, thick] (d1) -- (d2);
\draw[-, thick] (d0) -- (d3);
\draw[-, dotted, thick] (d1) -- (d3);
\draw[-, thick] (d3) -- (d2);
\draw[->-, thick] (d1) -- (d2);

\draw[dotted] (d1) -- (d3) node [midway, above=3.5pt, left=1.5pt, fill=none] {$e_1$};
\draw (d3) -- (d0) node [midway, above=5pt, fill=white] {$e_2$};
\draw (d2) -- (d3) node [midway, right=5pt, fill=white] {$e$};
\draw (d0) -- (d1) node [midway, left=5pt, fill=white] {$e_3$};
\draw (d2) -- (d1) node [midway, below=5pt, fill=white] {$e_4$};
\end{tikzpicture}
\endminipage\hfill
\caption{The pattern of each Meek rule as in the second part of the proof of proposition~\ref{prop:R_inner1}. An $X$ denotes that an intervention occured at that node. In our tree diagrams, if both vertices of an edge are intervened on, to prevent the identification of both child edges (and hence the edge itself by the corresponding Meek rule), it must be that both children have both of their vertices intervened on.}
\label{fig:orienting_down_tree}
\end{figure}
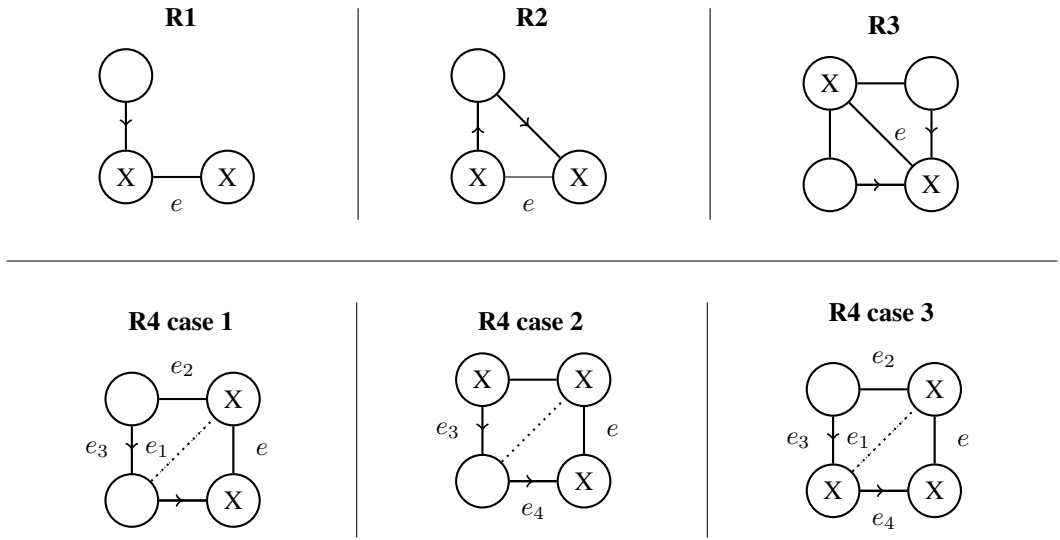

Now consider a different tree diagram for how $e$ is oriented in $R(I_2\cup \{v\})$. Clearly $v$ must be included as a vertex in at least one of the leaf edges, else the same pattern could allow us to orient $e$ using $I_2$. We show that if both vertices of an edge are intervened on, to prevent the identification of both child edges and hence the edge itself by the  Meek rules, it must be that both children have both of their vertices intervened on in $I_2$. This is shown in Figure~\ref{fig:orienting_down_tree}. 

In Figure~\ref{fig:orienting_down_tree}, \textbf{R1}, \textbf{R2} and \textbf{R3} are clear. For \textbf{R4 case 1}, $e_2$ must orient right to left (or get $e$ by R1). Also, $e_1$ must go towards the top right or we get $e$ by R2. Thus $e_3$ must point upwards by R2 which is a contradiction, since we know in the R4 pattern $e_3$ points downwards. Intervening on more edges to avoid this leads us to cases 2 and 3. For \textbf{R4 case 2}, we learn $e_3$ and $e_4$ and hence orient $e$. For \textbf{R4 case 3}, $e_2$ must point left else we get $e$ by R1. We know $e_3$ points down and it is oriented by R0. Then, $e_4$ goes left to right by R1. $E_1$ then points towards the bottom left by $R_2$ and hence $e$ is oriented by R2. 

Hence, there must be some path from the root $e$ to a leaf edge such that all members of the path have both of their vertices intervened on in $I_2$ and hence $I_2\cup\{v\}$. If $v\in e$, we have a contradiction since $v \notin I_2$. If $v \notin e$, then the leaf edge in our tree representation containing $v$ is not identified using R0 with $I_2+v$ either. Hence, this tree cannot possibly represent the sequence of Meek rules that lead to orienting $e$. Thus if $e \in R(I_2 \cup \{v\})$ and $e \notin R(I_2)$, then $e \notin R(I_1)$.
\end{proof}

\begin{proposition}
\label{prop:R_inner2}
$R(I_1) \setminus R(I_1 \cup \{v\}) \subseteq R(I_2) \setminus R(I_2\cup\{v\})$
\end{proposition}
\begin{proof}
As in the proof of lemma~\ref{lem:F_mon_sub}, we can write $R(I) = M(A(I))$, where $A$ returns edges oriented directly by the intervention and $M$ returns these in addition to any oriented due to Meek rules. 

By symmetry in the definition of $A$, we can see that $A(I^C) = A(I)$ and hence $R(I^C) = R(I)$. Take some edge $e \in R(I_1) \setminus R(I_1 \cup \{v\})$. Then by symmetry we have $e \in R(I_1^C) \setminus R(I_1^C \setminus \{v\})$. Then since $I_1^C \setminus \{v\} \supseteq \bar{I_2} - \{v\}$, by proposition~\ref{prop:R_inner1} we must have $e \in R(I_2^C) \setminus R(I_2^C \setminus \{v\})$. Again by symmetry we then have $e \in R(I_2) \setminus R(I_2 \cup \{v\})$. 
\end{proof}

We can conclude that $\Feo^\xi$ is submodular in an identical way to how we did in proving lemma~\ref{lem:F_mon_sub}: by combining propositions~\ref{prop:R_inner1} and ~\ref{prop:R_inner2}.

\subsection{Proof of Theorem~2}
At each iteration of selecting an intervention, Theorem~\ref{thm:mokhtari} lower bounds how close the marginal gain compared to the greedy intervention is. Let $\xi$  be the set of interventions \scg{} selects. Let $\xi^*$ be the optimal batch of interventions, and ${I^*}_i$ be the ith intervention in this set. Due to lemma~\ref{lem:F_mon_sub} (monotonicity), $\xi^*$ contains exactly $m$ interventions. $\{\xi_i\}_{i\geq 0}$ is the entire intervention set after each greedy selection. Define marginal improvement $\Delta(I | \xi) = \Feo(I \cup \xi) - \Feo(\xi)$. The following holds for all $i$:
\begin{equation*}
\begin{split}
\Feo(\xi^*) &\leq \Feo(\xi^* \cup \xi_i) \\
& \leq \Feo(\xi_i) + \sum_{j=1}^k \Delta({I^*}_j \mid \xi_i \cup \{{I^*}_1 ,..., {I^*}_{j-1} \}) \\
& \leq \Feo(\xi_i) + \sum_{I \in \xi^*} \Delta(I \mid \xi_i) \\
& \leq \Feo(\xi_i) +  \sum_{I \in \xi^*} e \E\left[ \Feo(\xi_{i+1}) - \Feo(\xi_i) \right] + e\epsilon \\
& \leq \Feo(\xi_i) +   em \E\left[ \Feo(\xi_{i+1}) - \Feo(\xi_i) \right] + em\epsilon
\end{split}
\end{equation*}
The first line is due to monotonicity. The second is a telescoping sum. The third is due to submodularity (lemma~\ref{lem:F_mon_sub}). The fourth is a result of Theorem~\ref{thm:mokhtari}, since what we write is lower bounded by the greedy choice, which by definition has greater marginal improvement than any other intervention. The expectation here is over noise in selecting the $(i+1)$th intervention. The final line just notes that there are $m$ elements in the sum.

Now define $\delta_i = \Feo(\xi^*) - \Feo(\xi_i)$. We rearrange the above to get

$$\delta_i \leq em( \delta_i - \E[\delta_{i+1}] + \epsilon),$$

where again the expectation is over selection of the $(i+1)$th intervention. 

Now we telescope this inequality, subbing in $i+1 = m$ to obtain our final result.

\begin{equation*}
\begin{split}
\E[\delta_{m}] &\leq \left(1 - \frac{1}{em} \right) \delta_{m-1} + \epsilon \\
&\leq \left(1 - \frac{1}{em} \right)^m \delta_0 + \sum_{j=0}^m \left(1 - \frac{1}{em} \right)^j \epsilon \\
&\leq e^{-1/e} \delta_0 + \sum_{j=0}^m e^{-\frac{j}{em}} \epsilon \\
\end{split}
\end{equation*}

The third line uses $1-x \leq e^{-x}$. The final term on the right hand side of the last line reduces to $\sigma \epsilon$, where 

$$\sigma = \frac{e^{-1/e} - e^{1/(em)}}{1-e^{1/(em)}}.$$

By noting that $\delta_0 = \Feo(\xi^*)$, we get

$$\E[\Feo(\xi)] \geq \left( 1 - e^{-\frac{1}{e}}\right) \Feo(\xi^*) - \epsilon \sigma.$$

Given Theorem ~\ref{thm:mokhtari} and that we select $m$ greedy interventions, this requires $\bigO \left( m p^{5/2}/\epsilon^3 \right)$ calls to $R$. We remove the dependence on $m$ in $\sigma$ into the runtime by noting that $\sigma = \mathcal{O}(m)$ and hence that if we do $\bigO \left( m^4 p^{5/2}/\epsilon^3 \right)$ calls to the Meek rules, we get 

$$\E[\Feo(\xi)] \geq \left( 1 - e^{-\frac{1}{e}}\right) \Feo(\xi^*) - \epsilon.$$

\subsection{Proof of Lemma~3}
\citet{agrawal2019abcd} prove monotone submodularity of \Fmi{}. \apFinf{} is the special case of this objective in the limit of infinite samples per intervention.

\subsection{Proof of Proposition~1}

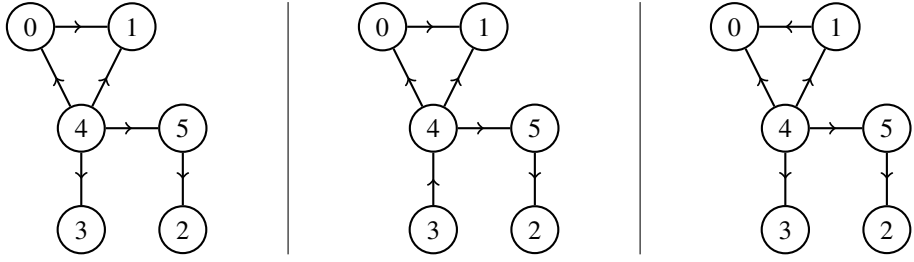
\begin{figure*}[t]
\vskip 0.2in
\minipage{0.33\textwidth}
\centering
\begin{tikzpicture}[scale = 0.9, roundnode/.style={circle, draw=black!100, fill=black!100, thick, minimum size=2mm},
rootnode/.style={circle, draw=black!100, fill=black!0, thick, minimum size=2mm},
]
\node[rootnode] (d0) at (0, 7.5) {0};
\node[rootnode] (d1) at (1.5, 7.5) {1};
\node[rootnode] (d2) at (2.25, 4.5) {2};
\node[rootnode] (d3) at (0.75, 4.5) {3};
\node[rootnode] (d4) at (0.75, 6) {4};
\node[rootnode] (d5) at (2.25, 6) {5};

\draw[->-, thick] (d0) -- (d1);
\draw[->-, thick] (d4) -- (d1);
\draw[->-, thick] (d4) -- (d0);
\draw[->-, thick] (d4) -- (d3);
\draw[->-, thick] (d4) -- (d5);
\draw[->-, thick] (d5) -- (d2);
\end{tikzpicture}
\endminipage\hfill \vline
\minipage{0.33\textwidth}
\centering
\begin{tikzpicture}[scale = 0.9, roundnode/.style={circle, draw=black!100, fill=black!100, thick, minimum size=2mm},
rootnode/.style={circle, draw=black!100, fill=black!0, thick, minimum size=2mm},
]
\node[rootnode] (d0) at (0, 7.5) {0};
\node[rootnode] (d1) at (1.5, 7.5) {1};
\node[rootnode] (d2) at (2.25, 4.5) {2};
\node[rootnode] (d3) at (0.75, 4.5) {3};
\node[rootnode] (d4) at (0.75, 6) {4};
\node[rootnode] (d5) at (2.25, 6) {5};

\draw[->-, thick] (d0) -- (d1);
\draw[->-, thick] (d4) -- (d1);
\draw[->-, thick] (d4) -- (d0);
\draw[->-, thick] (d3) -- (d4);
\draw[->-, thick] (d4) -- (d5);
\draw[->-, thick] (d5) -- (d2);
\end{tikzpicture}
\endminipage\hfill \vline
\minipage{0.33\textwidth}
\centering
\begin{tikzpicture}[scale = 0.9, roundnode/.style={circle, draw=black!100, fill=black!100, thick, minimum size=2mm},
rootnode/.style={circle, draw=black!100, fill=black!0, thick, minimum size=2mm},
]
\node[rootnode] (d0) at (0, 7.5) {0};
\node[rootnode] (d1) at (1.5, 7.5) {1};
\node[rootnode] (d2) at (2.25, 4.5) {2};
\node[rootnode] (d3) at (0.75, 4.5) {3};
\node[rootnode] (d4) at (0.75, 6) {4};
\node[rootnode] (d5) at (2.25, 6) {5};

\draw[->-, thick] (d1) -- (d0);
\draw[->-, thick] (d4) -- (d1);
\draw[->-, thick] (d4) -- (d0);
\draw[->-, thick] (d4) -- (d3);
\draw[->-, thick] (d4) -- (d5);
\draw[->-, thick] (d5) -- (d2);
\end{tikzpicture}
\endminipage\hfill
\begin{center}
\caption{The 3 DAGS in $\tilde{\mathcal{G}}$ for the counterexample in proving Proposition~1.}
\label{fig:prop}
\end{center}
\vskip -0.2in
\end{figure*}

To show that $\apFinf^{\xi}$ is in general not submodular, we need to give a specific example of a constraint set, and set of DAGs $\tilde{\mathcal{G}}$. We let $m=1, k=4$. The set $\tilde{\mathcal{G}}$ is given in Figure~\ref{fig:prop}. We consider the set of interventions before carrying out experiments, $\xi'$ to be the empty set. To break the definition of submodularity as given in definition~2 we need to define an existing intervention to add, $I_2$, and a subset of this, $I_1$. For this example we choose $I_2 = [1,2,3]$ with nodes numbered as in Figure~\ref{fig:prop}. $I_1 = [1,2]$. We also need to choose a perturbation to add to each intervention, and we choose node $0$. 

The computation of the objective is carried out in \emph{proposition.py} in the accompanying code. 

Note that for the special case $\tilde{\mathcal{G}} = \mathcal{G}$, we have not constructed a counterexample. In fact, in this case we suspect that $\apFinf^{\xi}$ is submodular, but don't have a proof. If true, this may suggest that an algorithm similar to \scg{} could be used to maximize Objective~\ref{obj:infABCDobj} directly without approximating the MEC with a bag of DAGs. An additional difficulty for this objective, however, is the second potentially exponential sum required to compute essential graph sizes embedded within the logarithm.

\subsection{Proof of Theorem~3}
We know that $\apFinf(\mathcal{S}) = 0$, because the graph is then fully identified, meaning $\abs{\tilde{\textrm{Ess}}^{\xi \cup \xi'}(G)}$ is 1 for all $G$.  We also know that $\min(\apFinf) = \apFinf(\emptyset)$. The submodularity of $\apFinf$ over groundset $\mathcal{S} \subset \mathcal{I}$ (lemma~3) along with this boundedness of the function is used to get the final bound. \\

Say we greedily select $m$ members of $\mathcal{S}$ to construct $\xi_m$. We prove by induction that $Q(m) = \left(\apFinf(\xi_m) \geq  \left(1-\frac{m}{\abs{\mathcal{S}}}\right) \apFinf(\emptyset) \right)$ is true for all $m$ where $0 \leq m \leq \abs{S}$. \\

The base case $Q(0)$ is trivial. Now assume that $Q(m)$ is true. Since $\apFinf$ is submodular over groundset $\mathcal{S}$, it satisfies the diminishing returns property of definition~2. Therefore it must be the case that $\exists I \in \mathcal{S} \setminus \xi_m$ such that $\apFinf(\{I\} \cup \xi_m) \geq \frac{1}{\abs{\mathcal{S}} - m} \apFinf(\xi_m)$. Note that if this was not the case, because of submodularity, adding all of the remaining interventions in $S$ in sequence would give $\apFinf(\mathcal{S}) < 0$ which would be a contradiction. Therefore \\

\begin{equation*}
\begin{split}
\apFinf(\xi_{m+1}) &\geq \frac{1}{\abs{\mathcal{S}} - m} \apFinf(\xi_m) \\
&\geq \frac{1}{\abs{\mathcal{S}} - m} \left(1-\frac{m}{\abs{\mathcal{S}}}\right) \apFinf(\emptyset)\\ 
&=  \left(1-\frac{m+1}{\abs{\mathcal{S}}}\right) \apFinf(\emptyset)
\end{split}
\end{equation*}

which completes the induction. The final result follows by applying the lower-bound on $\abs{\mathcal{S}}$ given in \citet{shanmugam2015learning}. A bound for the graph-sensitive separating system in \citet{lindgren2018experimental} can also be obtained by plugging in their lower-bound on $\abs{\mathcal{S}}$.

The runtime can be seen by observing that for each intervention, we compare $\abs{\mathcal{S}} = \mathcal{O}(\frac{p}{q} \log p)$ possible interventions. For each we evaluate $R$ a total of $\mathcal{O}(\abs{\tilde{\mathcal{G}}})$ times to compute \apFinf{}. Computing $R$ for each $G\in \mathcal{G}$ is sufficient to compute \apFinf{} because $R$ outputs all of the oriented edges in a graph given an intervention and hence can determine if a graph is in a certain interventional MEC. Thus, the overall runtime is $\mathcal{O}(m \abs{\tilde{\mathcal{G}}} \frac{p}{q} \log p)$ evaluations of $R$. The construction of the separating system itself is efficient compared to the computation of the Meek rules required to evaluate $R$ \cite{shanmugam2015learning}. 

\subsection{Extension to Soft Interventions}

Our results can also be used to develop algorithms for the soft intervention setting. A hard intervention makes the value of a variable independent of its parents. However, a soft intervention modifies a variable's value whilst maintaining the dependence on its parents. A simple example of a soft intervention is adding a constant value to the intervened node. In a GRN, a soft intervention might correspond to a gene knockdown, where a gene's expression is reduced but not set to $0$. 

\begin{definition}[Soft intervention]
A soft intervention $I$ on nodes $X_T$ for all $i \in I$, adds the intervention variables $W_i$ as an extra direct cause of $X_i$.  
\end{definition}

The soft intervention setting is greatly simplified by the following result in \citet{ghassami2019interventional}.

\begin{lemma}[\citet{ghassami2019interventional}]
In the infinite sample setting, a set of $k$ soft interventions of size $1$ is equivalent to targeting the same $k$ nodes in a single soft intervention. By equivalent, we mean it identifies the same information regarding the true DAG.
\end{lemma}

Therefore, the problem of selecting up to $m$ interventions with size at most $q$ is equivalent to selecting $mq$ single node interventions in the hard intervention case. In \citet{ghassami2018budgeted}, the authors show that for this setting a greedy algorithm can achieve a constant factor guarantee for the objective \Feo{}, where the constant factor will be $1- \frac{1}{e}$ instead of $1-\frac{1}{e^{1/e}}$. An almost identical analysis to that done in Theorem~\ref{thm:eo} can show that this guarantee is also achieved for optimizing \apFinf{}. 

\subsection{Further Experiment Details}

\begin{figure*}[t]
\vskip 0.2in
\begin{center}
\minipage{\columnwidth/3}
\textbf{a}
\centering
\includegraphics[width=\columnwidth]{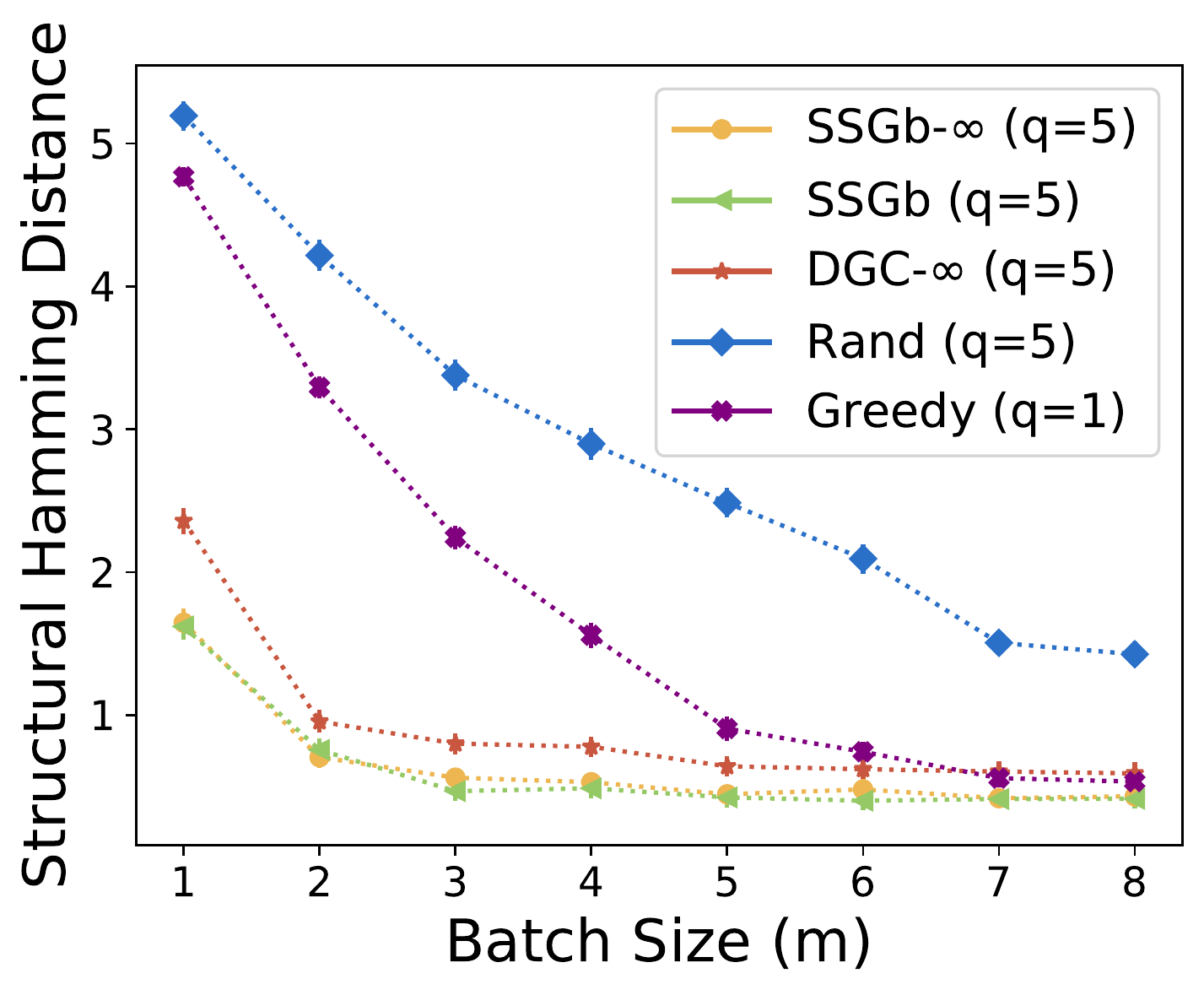}
\endminipage\hfill 
\minipage{\columnwidth/3}
\textbf{b}
\centering
\includegraphics[width=\columnwidth]{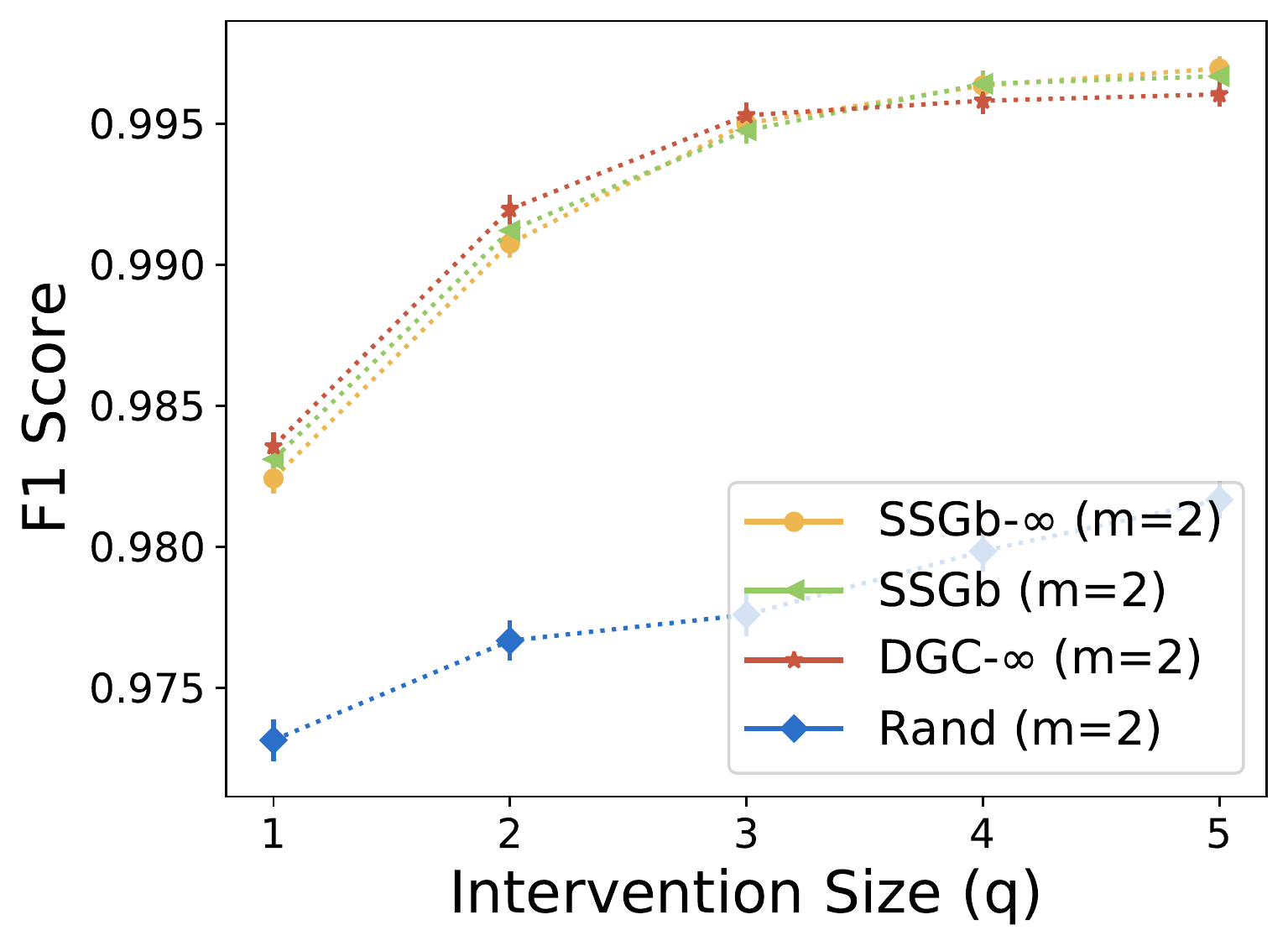}
\endminipage\hfill 
\minipage{\columnwidth/3}
\textbf{c}
\centering
\includegraphics[width=\columnwidth]{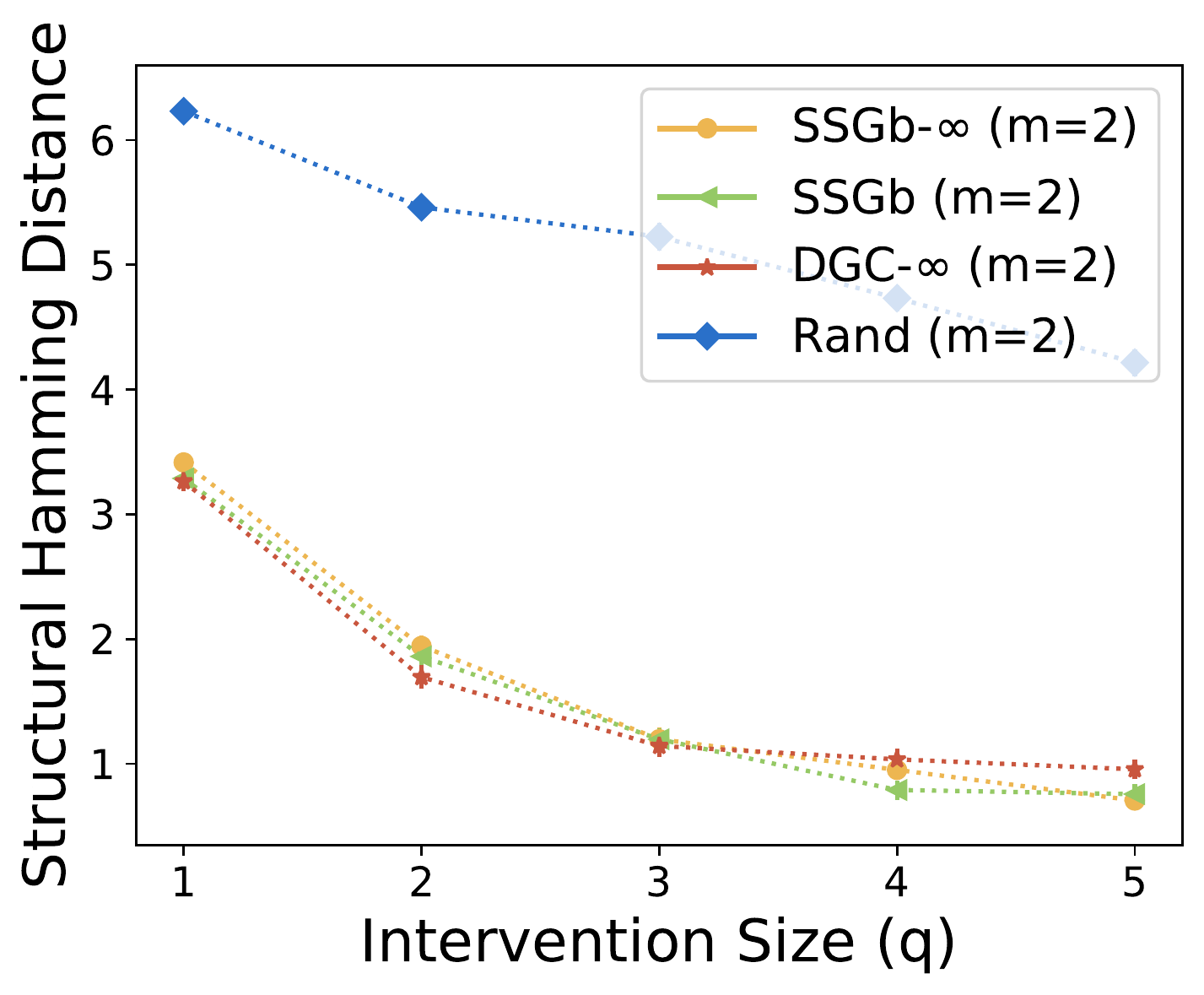}
\endminipage\hfill 

\caption{\textbf{a)} We measure performance on the finite sample task (same as in Figure~\ref{fig:inf}(e)) with fixed intervention size ($q=5$) using the SHD metric. The rank ordered performance of methods matches the outcome when we used the F1 metric. \textbf{b}, \textbf{c)} Here we consider the same experiment as (a) but plot F1 and SHD score respectively for changing intervention size with fixed batch size ($m=2$). In both cases our algorithms clearly outperform a random approach. 
} 
\label{fig:finite_sup}
\end{center}
\vskip -0.2in
\end{figure*}

All algorithms related to \ssg{} make use of lazy evaluation for speed-up, as in \citet{ghassami2018budgeted}. In greedy selection, lazy evaluation skips interventions which, based on previous evaluations and submodularity of the objective, could not be the greedy option. This results in no change in the selected interventions but reduces the number of necessary comparisons. 

When constructing graph agnostic separating systems according to the method of \citet{shanmugam2015learning}, we compute them exactly as specified since the algorithm is fast. For the graph sensitive construction method of \citet{lindgren2018experimental}, we use approximate algorithms for constructing a minimal vertex covering and a minimal graph-coloring, since both subroutines are NP-hard \cite{karp1972reducibility}. For graph-coloring, we use the Welsh-Powell algorithm which is near-optimal for graphs of bounded degree \cite{welsh1967upper}. For vertex cover, we use a 2-approximation algorithm based on greedily finding a maximal matching \cite{papadimitriou1991optimization}. 

For infinite sample experiments, when approximating the objective for use in our algorithms we use a multiset of 40 DAGs uniformly sampled from the MEC with replacement. However, for evaluation we use all DAGs in the MEC.

For \scg{}, when using Pipage rounding, we round $10$ times and select the intervention with greatest approximate objective value. 

For finite-sample experiments, the approximate prior over DAGs consists of 100 DAGs uniformly sampled from the MEC of the true DAG, or the MEC itself if the MEC has size less than 100. In the latter case, the rest of the 100 DAGs are given by bootstrapping the observational data and using the techniques of \citet{yang2018characterizing} to infer DAGs. 
In evaluating the finite-sample objective, there is some variance since we observe samples with noise. For evaluating methods on Objective~\ref{obj:abcdobj} as in Figure~\ref{fig:inf}(d), we average over 10 repeats. For \ssgb{}, which makes use of the objective for greedily selecting interventions, we approximate the objective by 10 repeats also. 
In Figure~\ref{fig:inf}(e) we plot F1 scores for predicting the presence of edges in the true DAG. We compute the probability of each directed edge being present by 

$$\Prob(u\rightarrow v) = \sum_{G \in \tilde{\mathcal{G}}} \hat{\Prob}(G) \mathbbm{1}((u \rightarrow v) \in G)$$

where $\tilde{\mathcal{G}}$ is the set of bootstrapped DAGs and $\hat{\Prob}$ is the posterior after collecting samples from interventions. A predicted graph is then estimated by thresholding the edge probabilities. We select the threshold that gives the graph with maximum F1 score, and then plot the F1 score. 

Similarly, we compute the weighted average (across the posterior) of structural hamming distances (SHDs) between graphs in the set of bootstrapped DAGs and the true DAG. This gives results similar to those of when we plot F1 score, as shown in Figure~\ref{fig:finite_sup}. Methods selecting the most informative interventions will have lower mean SHD because they will decrease the posterior probability of graphs distant from the true DAG. 

In Figure~\ref{fig:runtimes}(a-b) we compare the average runtimes of our algorithms running on the same hardware in a computing cluster. We do this for both the infinite and finite sample case. We see that although the \scg{} algorithm is slower than alternative approach on infinite samples, in finite samples it has a faster runtime than approaches that use the true finite sample objective. We also note that \ssgb{} has faster runtimes than \ssga{}, likely because the graph agnostic separating system construction returns larger separating systems. Whilst we aimed to maximize performance in terms of selecting the most informative experiments, these methods can be made faster at the expense of achieving lower objective values. For \ssg{} we could take fewer gradient steps, and for \scg{} we could use a smaller set of sampled DAGs to approximate the objective. 

In Figure~\ref{fig:runtimes}(c) we include plot performance in finite samples of the graph agnostic separating system for \ssg{} and see that its performance is lower than our alternative approaches. 

\begin{figure*}[t]
\vskip 0.2in
\begin{center}
\minipage{\columnwidth/3}
\textbf{a}
\centering
\includegraphics[width=\columnwidth]{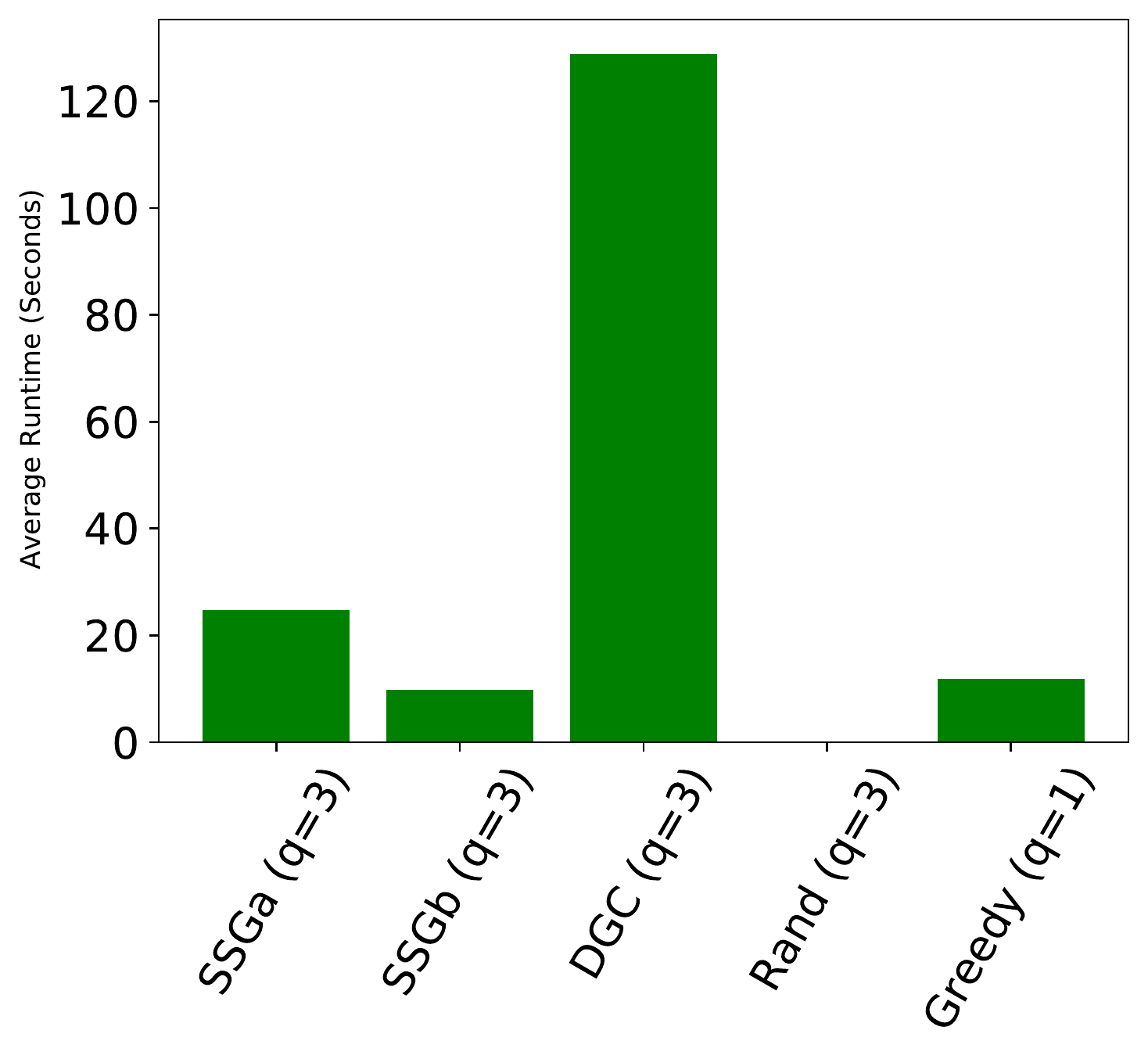}
\endminipage\hfill 
\minipage{\columnwidth/3}
\textbf{b}
\centering
\includegraphics[width=0.95\columnwidth]{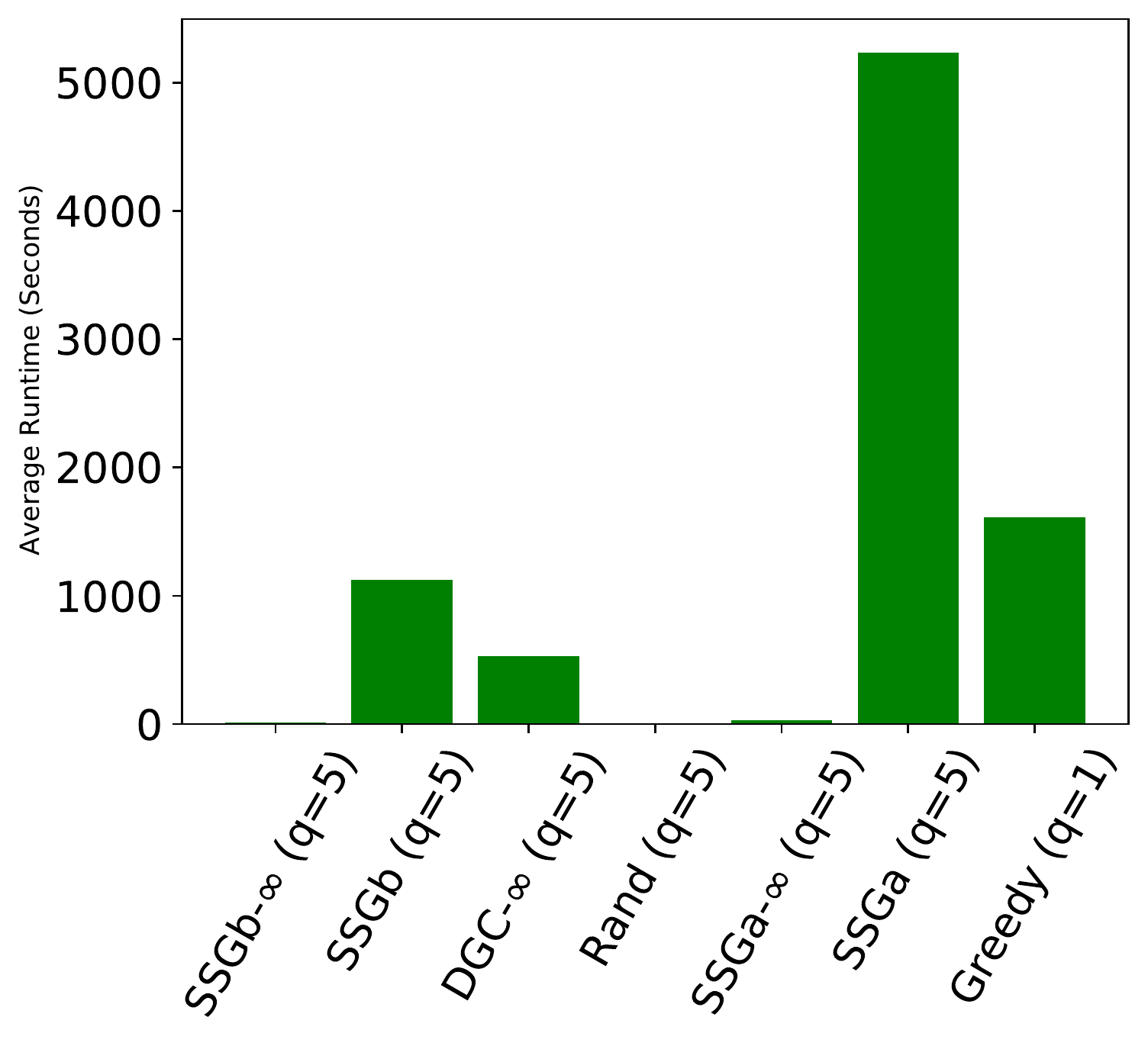}
\endminipage\hfill 
\minipage{\columnwidth/3}
\textbf{c}
\centering
\includegraphics[width=\columnwidth]{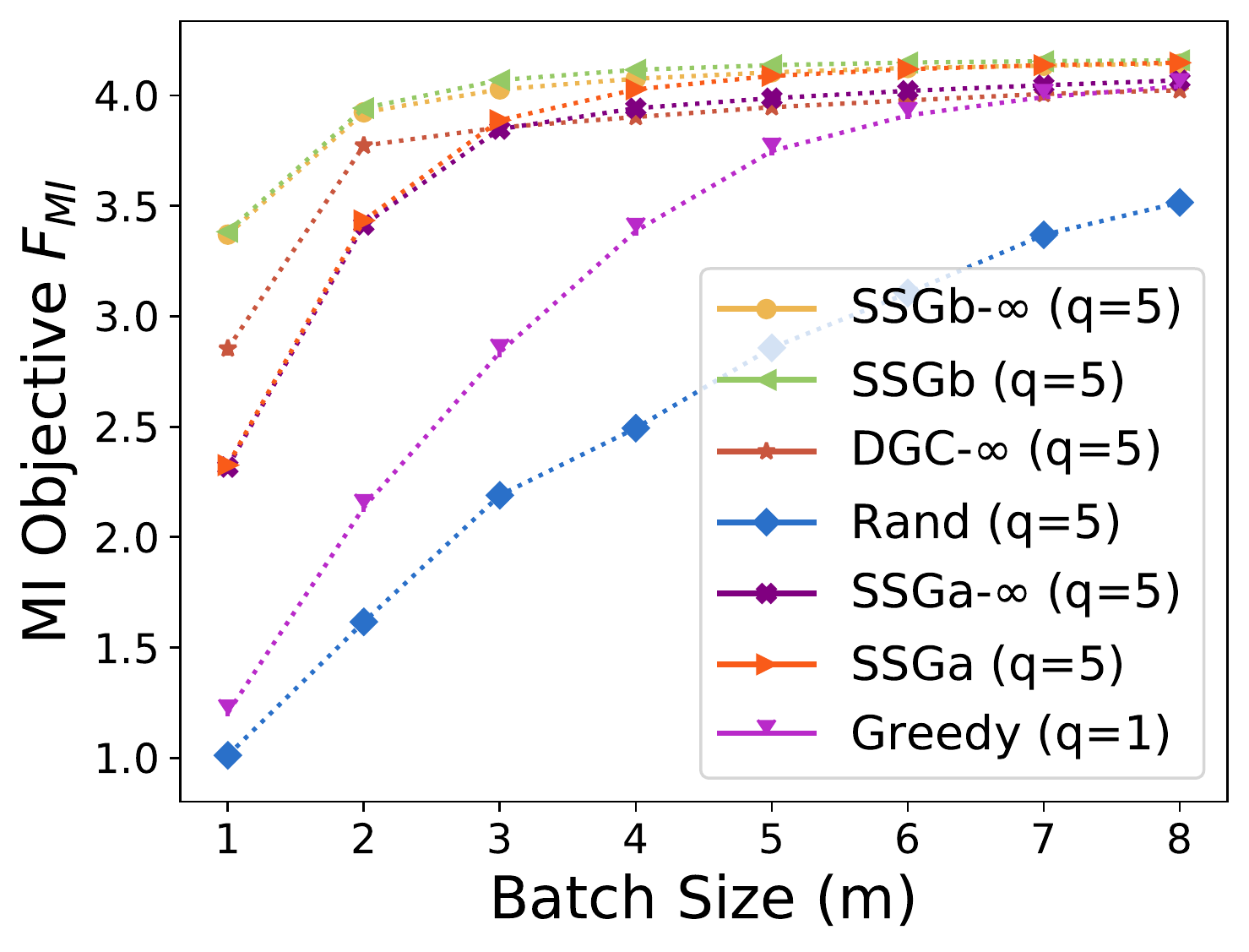}
\endminipage\hfill 

\caption{\textbf{a)} We demonstrate the runtime for selecting a batch of 5 interventions of size $q=3$ with each algorithm in the infinite sample setting. This is on ER $0.1$ graphs with $p=40$ nodes. We see that \scg{} is slower than other approaches but not impractical, whilst out of the \ssg{} methods, the graph-sensitive separating system construction results in faster runtimes. \textbf{b)} We show that for selecting $8$ interventions of size $q=5$, the infinite sample approximation methods we give result in substantially improved runtimes compared to methods that use the finite sample objective. This is again on the ER $0.1$ graphs but in the finite sample setting. \textbf{c)} We replicate Figure~\ref{fig:inf}(d) but include the \ssg{} algorithms based on the graph agnostic separating system construction. This separating system construction does not perform as well as the other algorithms we propose.} 
\label{fig:runtimes}
\end{center}
\vskip -0.2in
\end{figure*}

\subsection{Dream 3 Network Experiments}

\begin{figure*}[t]
\begin{center}
\minipage{\columnwidth/3}
\textbf{a}
\centering
\includegraphics[width=\columnwidth]{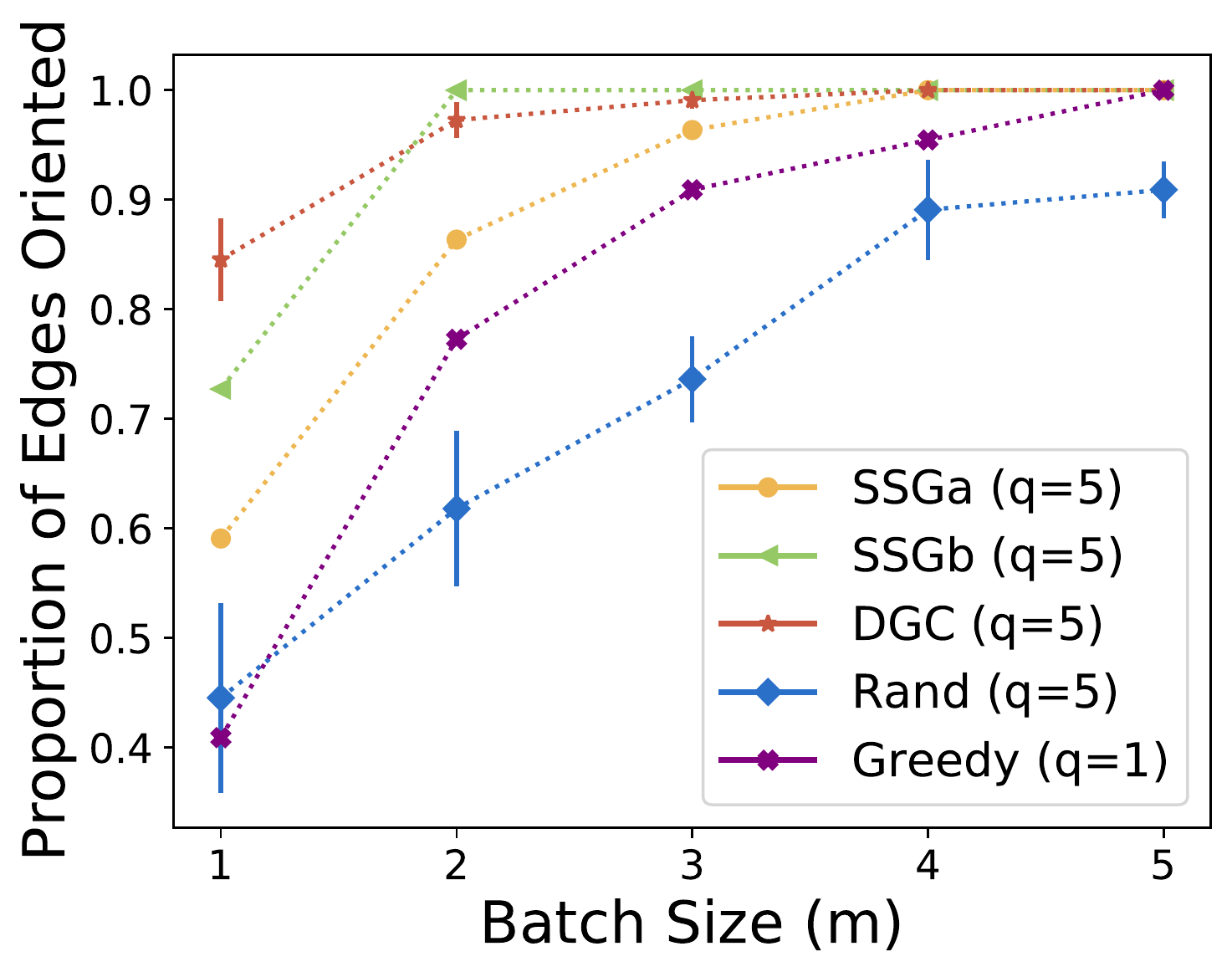}
\endminipage\hfill 
\minipage{\columnwidth/3}
\textbf{b}
\centering
\includegraphics[width=\columnwidth]{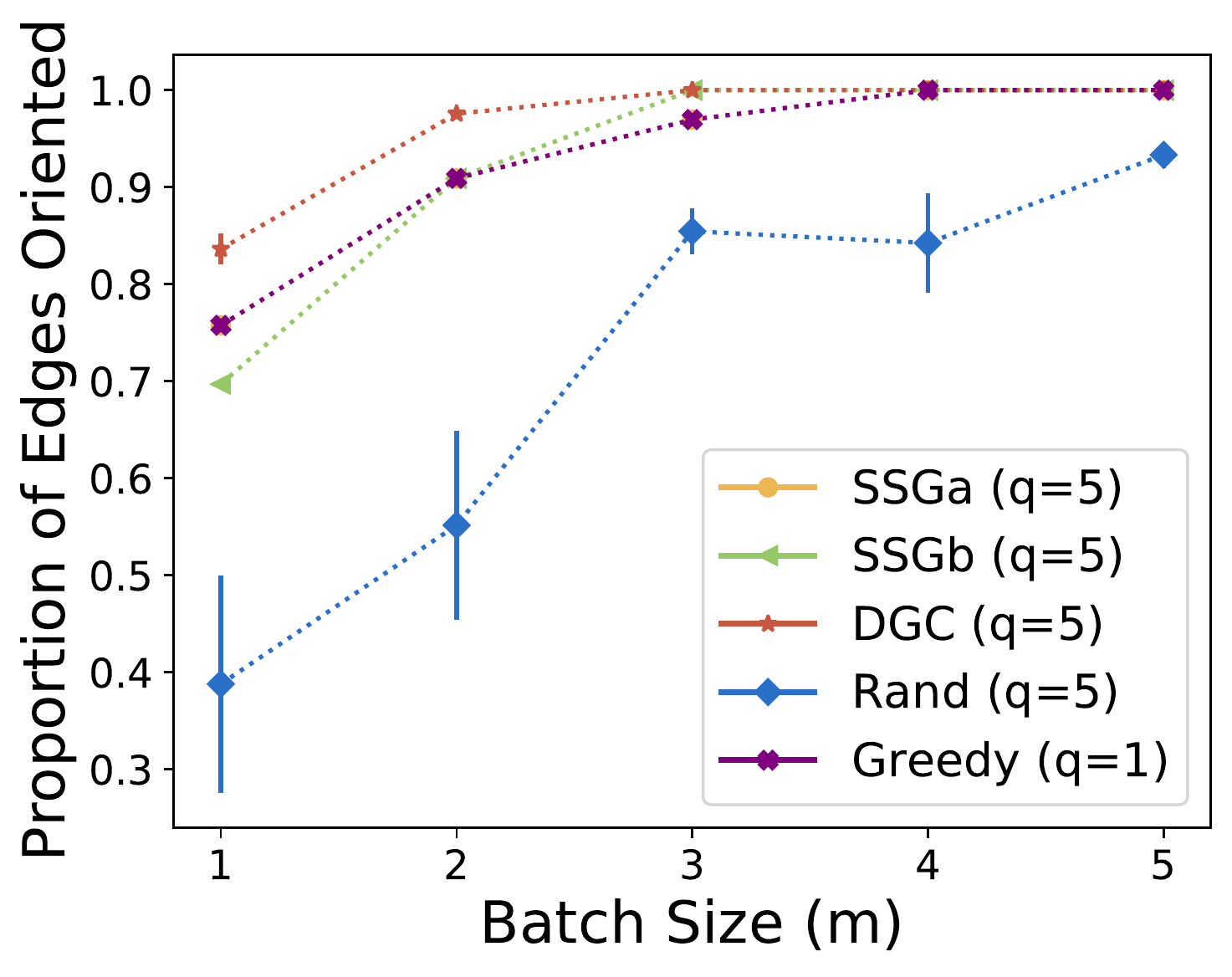}
\endminipage\hfill 
\minipage{\columnwidth/3}
\textbf{c}
\centering
\includegraphics[width=\columnwidth]{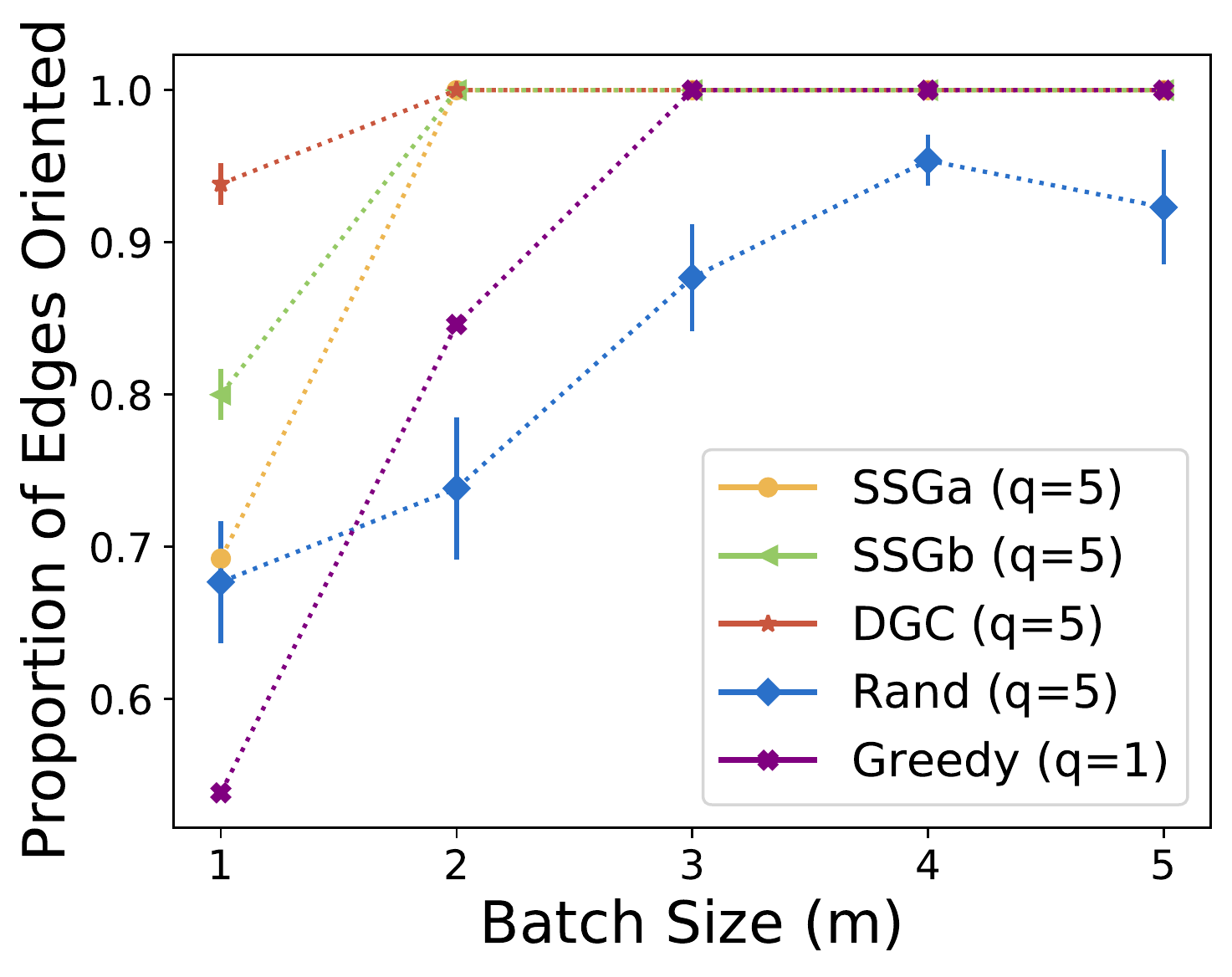}
\endminipage\hfill 
\minipage{\columnwidth/3}
\textbf{d}
\centering
\includegraphics[width=\columnwidth]{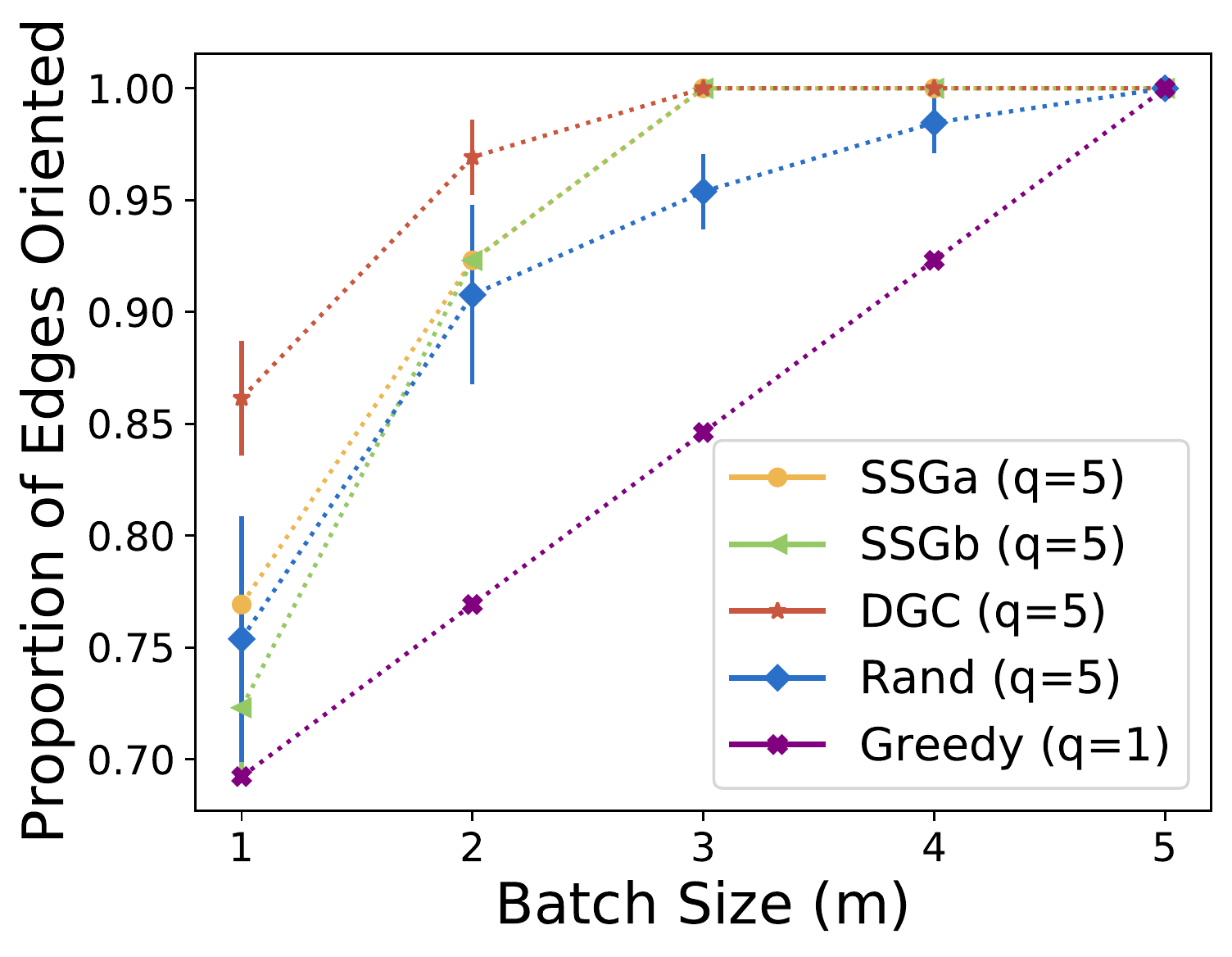}
\endminipage\hfill 
\minipage{\columnwidth/3}
\textbf{e}
\centering
\includegraphics[width=\columnwidth]{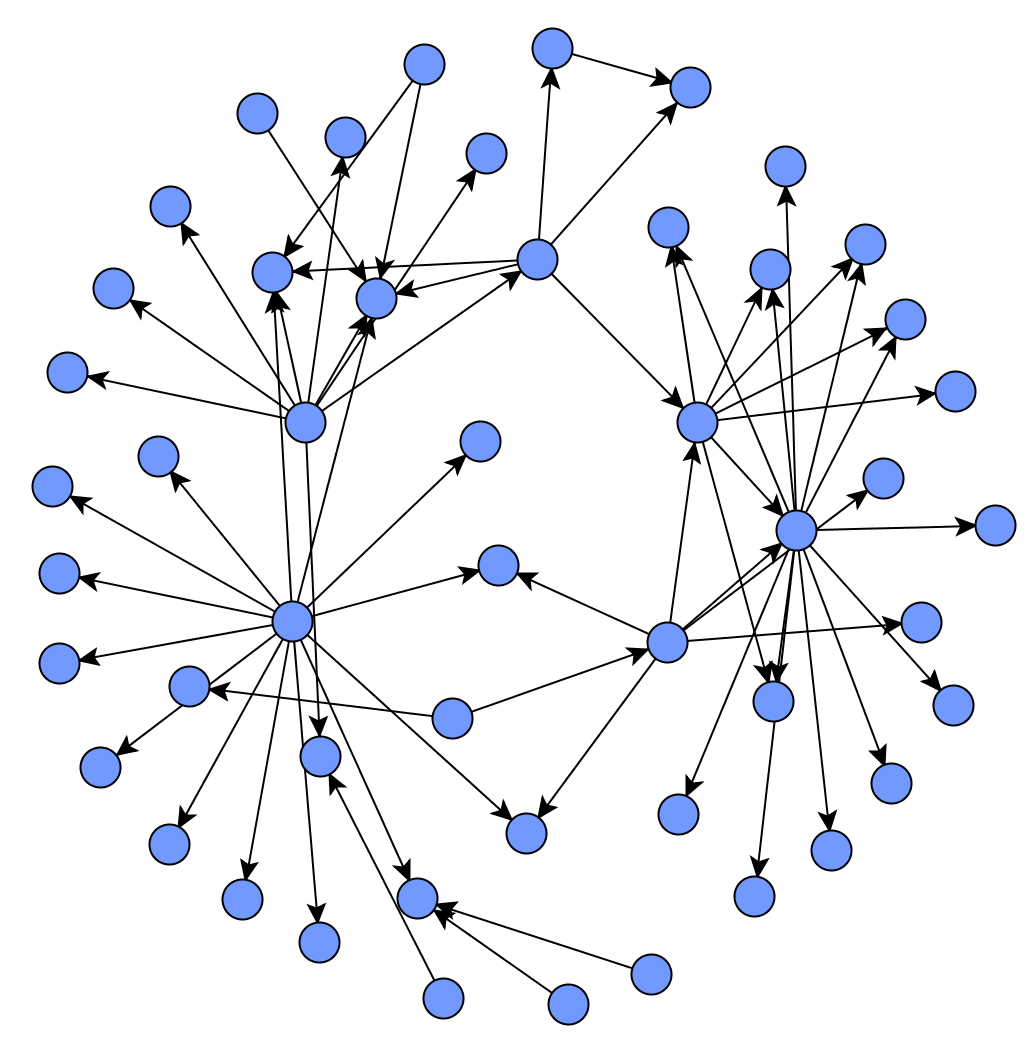}
\endminipage\hfill 
\minipage{\columnwidth/3}
\textbf{f}
\centering
\includegraphics[width=\columnwidth]{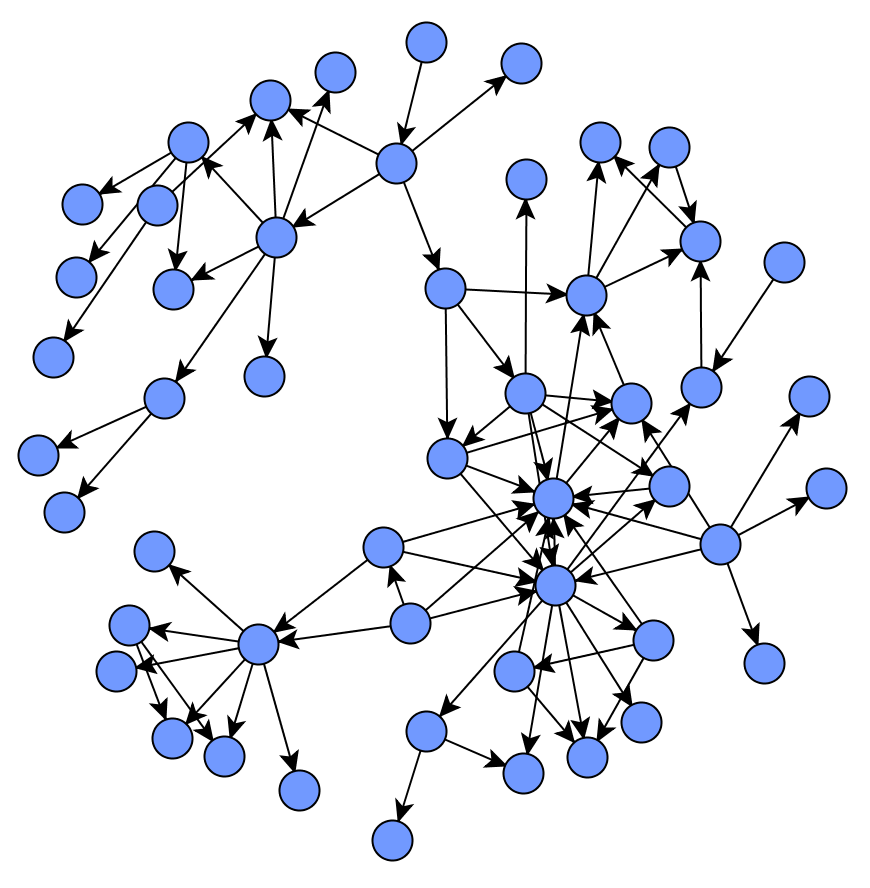}
\endminipage\hfill
\caption{\textbf{a--d)} The results for experiments identical to those in Figure~\ref{fig:inf} for the $p=50$ networks ``Ecoli1", ``Ecoli2", ``Yeast2" and ``Yeast3" from \citet{marbach2009generating} respectively. For an equal number of interventions and $q=5$, our methods in general orient more edges than both random and single-perturbation greedy interventions. \textbf{e, f)} The ground truth network for ``Ecoli1" and ``Yeast1`` respectively. 
} 

\label{fig:dream}
\end{center}
\vskip -0.2in
\end{figure*}

In Figure~\ref{fig:dream}(a--d) we give the results of DREAM3 networks not included in the main paper. For each, we record the number of edges oriented in the true DAG by each method, averaging over $5$ repeats. Again, we work in the infinite samples per intervention setting. In Figure~\ref{fig:dream}(e--f) we give illustrations of the ground truth DAG associated with two of the DREAM3 networks (images generated using software from \citet{marbach2009generating}). We note that these networks differ from typical Erd{\"o}s-Reny{\'i} random graphs in that they have a few nodes with many connections, and many nodes with sparser connections. 

\end{document}